%% file: OfflineDBNEval.tex
\documentclass{article}
\pdfoutput=1
\usepackage{times}
\usepackage{graphicx} %
\usepackage{subfigure} 
\usepackage{natbib}
\usepackage{algorithm}
\usepackage{algorithmic}
\usepackage{hyperref}

\usepackage[accepted]{icml2015} 

\usepackage{amsmath}
\usepackage{amssymb}
\usepackage{amsthm}
\usepackage{fullpage}

\usepackage{bibunits}
\defaultbibliographystyle{icml2015}

\usepackage{import}

\newtheorem{assume}{Assumption}
\newtheorem{cor}{Corollary}
\newtheorem{define}{Definition}
\newtheorem{lemma}{Lemma}

\newtheorem{prop}{Proposition}
\newtheorem{theorem}{Theorem}

\newcommand{\E}{\mathbb{E}} %
\newcommand{\bPi}{\pi_{b}} %
\newcommand{\ePi}{\pi} %
\newcommand{\uX}{\underline{X}} %
\newcommand{\uY}{\underline{Y}} %

\newcommand{\alg}{\texttt{G-SCOPE}}
\newcommand{\ks}{\texttt{KS}}
\newcommand{\flatAlg}{\texttt{Flat}}
\newcommand{\mfmc}{\texttt{MFMC}}
\newcommand{\cis}{\texttt{CIS}}

\newcommand{\sarsa}{\texttt{SARSA}}
\newcommand{\rmax}{\texttt{Rmax}}

\newcommand{\var}[1]{\underline{#1}}

\icmltitlerunning{Off-policy evaluation for MDPs with unknown structure}

\begin{document} 
\begin{bibunit}

\twocolumn[
\icmltitle{Off-policy Model-based Learning under Unknown Factored Dynamics}

\icmlauthor{Assaf Hallak}{ifogph@gmail.com}
\icmladdress{Technion,
            Haifa, Israel}
\icmlauthor{Fran\c{c}ois Schnitzler}{francois@ee.technion.ac.il}
\icmladdress{Technion,
            Haifa, Israel}
\icmlauthor{Timothy Mann}{mann@ee.technion.ac.il}
\icmladdress{Technion,
            Haifa, Israel}
\icmlauthor{Shie Mannor}{shie@ee.technion.ac.il}
\icmladdress{Technion,
            Haifa, Israel}
\icmlkeywords{Factored MDPs, Off-Policy Evaluation, Structure Learning}

\vskip 0.3in
]

\begin{abstract} 
Off-policy learning in dynamic decision problems is essential for providing strong evidence that a new policy is better than the one in use. But how can we prove superiority without testing the new policy? To answer this question, we introduce the \alg\ algorithm that evaluates a new policy based on data generated by the existing policy. Our algorithm is both computationally and sample efficient because it greedily learns to exploit factored structure in the dynamics of the environment. We present a finite sample analysis of our approach and show through experiments that the algorithm scales well on high-dimensional problems with few samples.
\end{abstract}

\section{Introduction} \label{Sec:Introduction}

Reinforcement Learning (RL) algorithms learn to maximize rewards by analyzing past experience with an unknown environment. Most RL algorithms assume that they can choose which actions to explore to learn quickly. However, this assumption leaves RL algorithms incompatible with many real-world business applications.

To understand why, consider the problem of on-line advertising: Each customer is successively presented with one of several advertisements. The advertiser's goal is to maximize the probability that a user will click on an ad. This probability is called the Click Through Rate (CTR, \citealt{richardson2007predicting}). A marketing strategy, called a policy, chooses which ads to display to each customer. However, testing new policies could lose money for the company. Therefore, management would not allow a new policy to be tested unless there is strong evidence that the policy is not worse than the company's existing policy. In other words, we would like to estimate the CTR of other strategies using only data obtained from the company's existing policy. In general, the problem of determining a policy's value from data generated by another policy is called {\em off-policy evaluation}, where the policy that generates the data is called the {\em behavior policy}, and the policy we are trying to evaluate is called the {\em target policy}. This problem may be the primary reason batch RL algorithms are hardly used in applications, despite the maturity of the field.

A simple approach to off-policy evaluation is given by the \mfmc\ algorithm \cite{Fonteneau2010}, which constructs complete trajectories for the target policy by concatenating partial trajectories generated by the behavior policy. However, this approach may require a large number of samples to construct complete trajectories. One may think that the number of samples is of little importance, since Internet technology companies have access to millions or billions of transactions. Unfortunately, the dimensionality of real-world problems is generally large (e.g., thousands or millions of dimensions) and the events they want to predict can have extremely small probability of occurring. Thus, sample efficient off-policy evaluation is paramount.

An alternative way of looking at the problem is through counterfactual (CF) analysis \cite{Bottou2013}. Given the outcome of an experiment, CF analysis is a framework for reasoning about what would have happened if some aspect of the experiment was different. In this paper, we focus on the question: what would have been the expected reward received for executing the target policy rather than the behavior policy?  One approach that falls naturally into the CF framework is Importance Sampling (IS) \cite{Bottou2013,Li2014}. IS methods evaluate the target policy by weighting rewards received by the behavior policy. The weights are determined by the probability that the target policy would perform the same action as the one prescribed by the behavior policy. Unfortunately, IS methods suffer from high variance and typically assume that the behavior policy visits every state that the target policy visits with nonzero probability. 

Even if this assumption holds, IS methods are not able to exploit structure in the environment because their estimators do not create a compact model of the environment. Exploiting this structure could drastically improve the quality of off-policy evaluation with small sample sizes (relative to the dimension of the state-space).
Indeed, there is broad empirical support that model-based methods are more sample efficient than model-free methods \cite{Hester2009,Jong2007}. However, one broad class of compact models are Factored-state Markov Decision Processes (FMDPs, \citealt{kearns1999efficient,Strehl2007,chakraborty2011structure}). An FMDP model can often be learned with a number of samples {\em logarithmic} in the total number of states, if the structure is known. Unfortunately, inferring the structure of an FMDP is generally computationally intractable for FMDPs with high-dimensional state-spaces \cite{chakraborty2011structure}, and in real-world problems the structure is rarely known in advance.

Ideally, we would like to apply model-based methods to off-policy evaluation because they are generally more sample efficient than model-free methods such as \mfmc\ and IS. In addition, we want to use algorithms that are computationally tractable. To this end, we introduce \alg, which learns the structure of an FMDP greedily. \alg\ is both sample efficient and computationally scalable. Although \alg\ does not always learn the true structure, we provide theoretical analysis relating the number of samples to the error in evaluating the target policy. Furthermore, our experimental analysis demonstrates that \alg\ is significantly more sample efficient than model-free methods.

The main contributions of this paper are:
\begin{itemize}
\item a novel, scalable method for off-policy evaluation that exploits unknown structure,
\item a finite sample analysis of this method, and
\item a demonstration through experiments that this approach is sample efficient.
\end{itemize}

The paper is organized as follows. In Section \ref{Sec:Background}, we describe the problem setting and notations. Section \ref{Sec:Algorithm} elaborates on our greedy structure learning algorithm. Our main theorem and its analysis are given in Section \ref{Sec:Analysis}. Section \ref{Sec:Experiments} presents experiments. In Section \ref{Sec:Discussion}, we discuss limitations of \alg\ and future research directions. 

\section{Background} \label{Sec:Background}
We consider dynamics that can be represented by a Markov Decision Process (MDPs; \citealt{puterman2009markov}):
\begin{define} \label{Def:MDPs}
A Markov Decision Process (MDP) is a tuple $(S, A, P(s'|s,a), R(s,a), \rho)$ where $S$ is the state space, $A$ is the action space, $P$ represents the transition probabilities from every state-action pair to another state, $R$ represents the reward function fitting each state-action pair with a random real number, and $\rho$ is a distribution over the initial state of the process. 
\end{define}
We denote by $\pi$ a Markov policy that maps states to a distribution over actions. The process horizon is $T$, and applying a policy for $T$ steps starting from $s_0 \sim \rho$ results in a cumulative reward known as the value function: $V^{\pi}(s_0) = \E \left[ \sum_{t=0}^{T-1} R(s_t, a_t) | s_0, \pi  \right]$, where the expectation is taken with respect to $P, R$ and $\pi$. We assume  $R$ is known and immediate rewards are bounded in $[0,1]$.

The system dynamics is as follows: First, an initial state $s_0$ is sampled from $\rho$. Then, for each time step $t=0,\ldots,T-1$, an action $a_t$ is sampled according to the policy $\pi(s_t)$, a reward $r_{t}$ is awarded according to $R(s_t, a_t)$ and the next state $s_{t+1}$ is sampled by $\Pr(\cdot|s_t, a_t)$. The quantity of interest is the expected policy value $\nu^{\pi} = \rho ^\top V^{\pi}$. 

\subsection{Off-Policy Evaluation}
We consider the finite horizon batch setup. Given are $H$ trajectories of length $T$ sampled from an MDP with an initial state distribution $\rho$ and behavior policy $\bPi$. 
The off-policy evaluation problem is to estimate the $T$-step value of a target policy $\ePi$ (different from $\bPi$). For the target policy $\ePi$, we aim to minimize the difference between the true and estimated policy value:
\begin{equation} \label{Eq:Evaluation Difference}
| \nu^{\ePi} -  \hat{\nu}^{\ePi} | .
\end{equation}

\subsection{Factored MDPs}

Suppose the state space can be decomposed into $D$ discrete values. We denote the $i^{\rm th}$ variable of $\uX$ by $\uX(i)$, and for a given subset of indices $\Psi \subseteq \left[ D \right] \triangleq \{1,2,..,D\}$, let $\uX(\Psi)$ be the subset of corresponding variables $\{ \uX(i) \}_{i \in \Psi}$. We define a factored MDP, similar to \citealt{guestrin2003efficient}: 

\begin{define} \label{Def:FMDPs}
A Factored MDP (FMDP) is an MDP $(S, A, P, R, \rho)$ such that the state $\uX \in S$ is composed of a set of $D$ variables $\{\uX(i)\}_{i=1} ^D$, where each variable can take values from a finite domain, such that the probability of the next state $\uY$ given that action $a$ is performed in state $\uX$ satisfies
\begin{equation} \label{Eq:Independency}
\Pr(\uY | \uX, a) = \prod_{i=1}^D \Pr(\uY(i) | \uX, a) \enspace .
\end{equation}
\end{define}

For simplicity, we assume that all variables lie in the same domain $\Gamma$, i.e., $\uX \in \Gamma^D$, where $\Gamma$ is a finite set. Furthermore, each variable in the next state $\uY(i)$ only depends on a subset of variables $\uX(\Phi_i)$ where $\Phi_i \subseteq [D]$. The indices in $\Phi_i$ are called the parents of $i$. When the size of the parent sets are smaller than $D$, then the FMDP can be represented more compactly:
\begin{equation} 
\Pr(\uY | \uX, a) = \prod_{i=1}^D \Pr(\uY(i) | \uX(\Phi_i), a) \enspace .
\end{equation}

Before delving into the algorithm and the analysis, we provide some notation. For a subset of indices $\Psi \subseteq [D]$, a realization-action pair $(v,a) \in \Gamma^{|\Psi|} \times A$ is a specific instantiation of values for the corresponding variables $\uX(\Psi), a$. We denote by $F_i = \Gamma^{ |\Phi_i| } \times A$ the set of all realization-action pairs for the parents of node $i$, and mark $\Lambda = \bigcup_{i=1}^D F_i$. 

The following quantities are used in the algorithm and consecutive analysis: Denote by $\Psi  \subseteq \left[ D \right] $ a subset of indices and by $v \in \Gamma^{|\Psi|}$ a realization of the corresponding variables:
\begin{multline} \label{Eq:Probabilities Notation}
\Pr(\uY(i) | \uX(\Psi) = v, a) \triangleq \\
 \shoveright{\frac{\sum_{t=1}^{T} \Pr(\uY(i), \uX(\Psi) = v, a, t )}{\sum_{t=1}^{T} \Pr(\uX(\Psi) = v, a, t )}} \\
\shoveleft{ \widehat{\Pr}(\uY(i) = y| \uX(\Psi) = v, a) \triangleq \frac{n(y, v, a)}{n(v, a)} \enspace ,}
\end{multline}
where the probabilities in the right term of the first equation are conditioned on the behavior policy $\bPi$ omitted for brevity. Note that if $\Psi \supseteq \Phi_i$ then $\Pr(Y(i) | X(\Psi)
= v, a)  = \Pr(Y(i) | X(\Phi_i) = v(\Phi_i), a)$, and the policy dependency cancels out.

\subsection{Previous Work}

Previous works on FMDPs focus on finding the optimal policy. Early works assumed the dependency structure is known \cite{guestrin2002algorithm,kearns1999efficient}. \citet{degris2006learning} proposed a general framework for iteratively learning the dependency structure (this work falls within this framework), yet no theoretical results were presented for their approach. SLF-Rmax \cite{Strehl2007}, Met-Rmax \cite{diuk2009adaptive} and LSE-Rmax \cite{chakraborty2011structure} are algorithms for learning the complete structure. Only the first two require as input the in-degree of the DBN structure. The sample complexity of these algorithms is {\em exponential} in the number of parents. Finally, learning the structure of DBNs with no related reward is in itself an active research topic \cite{Friedman:1998aa,Trabelsi:2013aa}.

There has also been increasing interest in the RL community regarding the topic of off-policy evaluation. Works focusing on model-based approaches mainly provide bounds on the value function estimation error. For example, the simulation lemma \cite{Kearns2002a} can be used to provide sample complexity bounds on such errors. On the other hand, model free approaches suggest estimators while trying to reduce the bias. \citet{precup2000eligibility} presents several methods based on applying importance sampling on eligibility traces, along with an empirical comparison; \citet{Theocharous2015offPolicyConfidence} had analyzed bounds on the estimation error for this method. A different approach was suggested by \citet{Fonteneau2010}: evaluate the policy by generating artificial trajectories - a concatenation of one-step transitions from observed trajectories. The main problem of these approaches besides the computational cost is that a substantial amount of data required to generate reasonable artificial trajectories.

\section{Algorithm} \label{Sec:Algorithm}
In general, inferring the structure of an FMDP is exponential in $D$ \cite{Strehl2007}. Instead, we propose a naive greedy algorithm which under some assumptions can be shown to provide small estimation error on the transition function (\alg\ - Algorithm \ref{Alg:G-SCOPE}). 

\begin{algorithm}                      %
\caption{\alg($H$ $T$-length traj., $\epsilon, \delta, C_2 =0$)}          %
\label{Alg:G-SCOPE}                           %
\begin{algorithmic}                    %
    \FOR{$i=1$ to $D$}
    	\STATE $\hat{\Phi}_i \Leftarrow \emptyset$
	    \REPEAT	    
	    \STATE $\Theta_i \Leftarrow \{ (v, v(j), a) \in \Gamma^{|\hat{\Phi}_i|+1} \times A :  j \in [D] \setminus$ \\ 			$\quad\hat{\Phi}_i, | n(v, v(j), a) | >  N(\epsilon, \delta)  \}$
	    \STATE For $N(\epsilon, \delta) =\frac{2\Gamma^2}{\epsilon^2} \ln \left( \frac{2\Gamma}{\delta_1}\right)$
	    \IF{$| \Theta | = 0$}
	    \STATE Break
	    \ENDIF 
	    \FOR{$j=1$ to $D$}
	   	\STATE $\text{diff}_j \Leftarrow \max_{ (v, v(j), a) \in \Theta} $
	    	\STATE $\;\;\; \| \widehat{\Pr}(Y(i) | X(\hat{\Phi}_i \cup j) = (v, v(j)), a) $
	    	\STATE $\;\;\; \;\;\; - \widehat{\Pr}(Y(i) | X(\hat{\Phi}_i)=v, a)  \|_1$
	    \ENDFOR
	    \STATE $j^* \Leftarrow {\arg\max}_{j \in [D]} \text{diff}_j$
	    \IF{$\text{diff}_{j^*} > C_2 + 2\epsilon$}
	    \STATE $\hat{\Phi}_i \Leftarrow \hat{\Phi}_i \cup j^*$
	    \ENDIF
	    \UNTIL{$\text{diff}_{j^*} \leq C_2 + 2\epsilon $}
    \ENDFOR    
    \STATE {\bf return} $\{  \hat{\Phi}_i \}_{i=1}^{D}$
\end{algorithmic}
\end{algorithm}

\alg\ (Greedy Structure learning of faCtored MDPs for Off-Policy Evaluation) receives off-line batch data, two confidence parameters $\epsilon$, $\delta$ and a minimum acceptable score $C_2$. The outputs $\hat{\Phi}_i$ are the estimated parents of each variable $i$. In the inner loop, the set $\Theta$ is defined as the set of all realization-action pairs which had been observed at least $N(\epsilon, \delta)$ times; These are the only pairs further considered. We then greedily add to $\hat{\Phi}_i$ the $j$'th variable which maximizes the $L_1$ difference between the old distribution depending only on $\hat{\Phi}_i$, and a distribution conditioned on the additional variable as well. Parents are no longer added when that difference is small, or when all possible realizations were not observed $N(\epsilon, \delta)$ times. The computational complexity of a naive implementation is $O(HT \Gamma D^2)$, since \alg\ sweeps the data for every input and output variable.

The main idea beyond \alg\ is that having enough samples will result in an adequate estimate of the conditional probabilities. Then, under appropriate regularity assumptions (stated in Section \ref{Sec:Analysis}), adding a non parent variable is unlikely. If parents have a higher effect than non-parents on the $L_1$ distance and non-parents have a weak effect, the $\arg\max$ procedure will most likely return only parents. When all prominent parents were found, or when there is not enough data for further inference, the algorithm stops. Once the set of assumed parents is available, we can build an estimated model and simulate {\em any} policy.

An important property of the \alg\ algorithm is that it does not necessarily find the actual parents. Instead, we settle on finding a subset of variables providing probably approximately correct transition probabilities. As a result, the number of considered parents scales with data available, a desired quality linking the model and sample complexity. Since we do not necessarily detect all parents, non-parents can have a non-zero influence on the target variable after all prominent parents have been detected. To avoid including these non-parents, the threshold to add a parent is $C_2$ plus some precision parameters. In practice, we use $C_2=0$ because including non-parents with an indirect influence on $\var{Y}(i)$ may improve the quality of the model.
However, in our analysis, we present Assumptions under which the true parents can be learned and explain $C_2$.

Finally, \alg\ can be modified to encode and construct the conditional probability distributions using decision trees. A different decision tree is constructed for each action and variable in the next state. Tree based models can produce more compact representations of the model than encoding the full conditional probability tables specified by $\hat{\Phi}_i$. While we analyze \alg\ as an algorithm that separates structure learning from estimating the conditional probability tables, for simplicity and clarity, in our experiments, we actually use a decision tree based algorithm. The modifications to the analysis for the tree based algorithm would add unnecessary complexity and distract from the key points of the analysis.

\section{Analysis} \label{Sec:Analysis}
By using a scalable but greedy approach to structure learning rather than a combinatorially exhaustive one, \alg\ can only learn arbitrarily well a subclass of models. In this section, we introduce three assumptions on the FMDP that describe this subclass, and then analyze the policy evaluation error for this subclass.

We divide $\Phi_i$ into non-overlapping ``weak'' ($\Phi_i^w$) and ``strong'' ($\Phi_i^s$) parents. These subsets will be defined formally later, but intuitively, parents in $\Phi_i^s$ have a large influence on $\var{Y}(i)$ and are easy to detect while parents in $\Phi_i^s$ have a smaller influence that may be below the empirical noise threshold and hence not be detected. Our assumptions state that (\ref{Ass:Coupling}) ``strong'' parents are sufficiently better than non-parents to be detected by \alg\ before non-parents; (\ref{Ass:dwarfNonParents}) conditionally on ``strong'' parents, non-parent have too little influence on $\var{Y}(i)$ to be accepted by \alg\ and (\ref{Ass:Submodularity}) conditioning on some ``weak'' parents does not increase the influence of other ``weak'' parents. The first two assumptions are used to bound the probability that \alg\ adds non parents in $\hat{\Phi}_i$ or does not add some strong parents, the last one to bound the error caused by the potential non-detection of weak parents.

\begin{assume} \label{Ass:Coupling}
	\textbf{Strong parent {superiority}.}
	For every $i\in [D]$, 
	there exists a ``strong'' subset of parents $\Phi^s_i \subseteq \Phi_i $
	such that 
	$\forall \Psi \subset \Phi_i$, 
	$ \Phi_i^s \backslash \Psi \neq \emptyset$,
	$j \in D \backslash \Phi_i$,
	$(v,v(j),a) \in \Gamma^{|\Psi \cup \{j\}|} \times A,$
	there exists
	$k \in \Phi^s_i \backslash \Psi$,
	such that 
	$\forall (v',v'(k),a')\in \Gamma^{|\Psi \cup \{k\}|} \times A :$
	for some $C_1 \geq 0$,
	\begin{equation}
	\begin{split}	
	& \| \Pr(\uY(i) | \uX(\Psi \cup \{ k \}) = (v', v'(k)), a')
	\\
	& \;\;\;\;\; - \Pr(\uY(i) | \uX(\Psi) = v', a')  \|_1  \geq 
	\\
	& \;\;\;\;\; \| \Pr(\uY(i) | \uX(\Psi \cup \{ j \}) = (v, v(j), a)
	\\
	& \;\;\;\;\;\;\;\; - \Pr(\uY(i) | \uX(\Psi) = v, a)  \|_1	+ C_1
	\enspace.
	\end{split}
	\end{equation}
\end{assume}

\begin{figure}[tbp]
\begin{center}
\def\svgwidth{0.8\columnwidth}
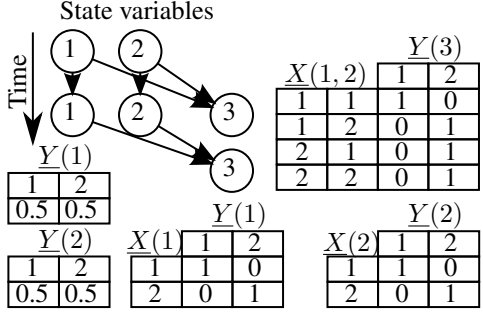
\caption{ An FMDP that fails to satisfy Assumption \ref{Ass:Coupling}. The factorization for a given action (not shown on the figure) is represented as a dynamic Bayesian network. States not relevant for the explanation are omitted. In the conditional transition probability tables, rows correspond to possible values of parent variables and columns to possible values of the variable. Cells at the intersection contain conditional probability values.}
\label{fig:ill_Coupling}
\end{center}
\end{figure}

Assumption \ref{Ass:Coupling} ensures that, in terms of influence on the conditional distribution of the target, \alg\ finds at least one ``strong'' parent variable $k$ more attractive than any non-parent variable $j$ as long as $ \Phi_i^s \backslash \hat{\Phi}_i \neq \emptyset$.
This prevents extreme cases where due to large correlation between parents and non-parents factors, large numbers of non-parents could be added before finding the actual parents, thus considerably increasing the sample complexity. $C_1$ quantifies how much more information a true parent will provide than non-parents. The larger $C_1$ the less likely \alg\ will add a non-parent in $\hat{\Phi}_i$.

Figure \ref{fig:ill_Coupling} illustrates a subset of the state variables and corresponding conditional transition probability distributions of an FMDP that, for the action implicitly considered, does not satisfy Assumption \ref{Ass:Coupling}. In this setting, for $t \geq 3$ and considering $\Psi= \emptyset$, we have
\begin{align*}
	 \| \Pr(\var{Y}(3) )
	 -   \Pr(\var{Y}(3) | \var{X}(3) = i)  \|_1 &= 2 \;\;\; \forall i \in \{1,2\}
\\
	 \| \Pr(\var{Y}(3) )
	 -   \Pr(\var{Y}(3) | \var{X}(1) = j)  \|_1 &= 1 \;\;\; \forall j \in \{1,2\}
	 .
\end{align*}
\alg\ would add $\var{X}(3)$, a non-parent, before any true parent of $\var{Y}(3)$ in the estimated parent set. Note that in this particular case it does not matter, as $\var{X}(3)$ perfectly determines $\var{Y}(3)$. However, adding noise in the transition probabilities would make $\var{X}(3)$ less accurate than $\var{X}(1)$ and $\var{X}(2)$ together. %

\begin{assume} \label{Ass:dwarfNonParents}
       \textbf{Non-parent conditional {weakness}.}
	For every $i\in [D]$,
	$\Phi^s_i$ as in Assumption \ref{Ass:Coupling},
	$\forall \Psi: \Phi^s_i \subseteq \Psi \subseteq \Phi_i$, 
	$j \in D \backslash \Phi_i$,	
	$(v,v(j),a) \in \Gamma^{|\Psi \cup \{j\}|} \times A:$
	for some $C_2 \geq 0$,
	\begin{equation}
	\begin{split}
	&\| \Pr(\uY(i) | \uX(\Psi \cup \{ j \}) = (v, v(j), a)
	\\
	& \;\;\; - \Pr(\uY(i) | \uX(\Psi) = v, a)  \|_1 \leq C_2
	\enspace.
	\end{split}
	\end{equation}
	
\end{assume}

Assumption \ref{Ass:dwarfNonParents} ensures that, after \alg\ has detected all strong parents, non-parents have low influence on the target variable and therefore \alg\ has a low probability to add them to $\hat{\Phi}_i$. If $\Phi_i^s = \Phi_i$, then $C_2=0$.

\begin{assume} \label{Ass:Submodularity}
       \textbf{Conditional {diminishing} returns.}
	There exists 
	$C_3 \geq 0$
	 such that for every 
	 $i\in [D]$, 
	 $\Phi^s_i$ as in Assumptions \ref{Ass:Coupling} and \ref{Ass:dwarfNonParents},
	 $\Psi: \Phi^s_i \subseteq \Psi \subseteq \Phi_i$, 
	 $j, k \in \Phi_i \setminus \Psi$, 
	 $(v, v(j), v(k), a) \in \Gamma^{|\Psi|+2} \times A$,
	 if %
	\begin{equation}
	\begin{split}
	& \| \Pr(\uY(i) | \uX(\Psi \cup \{ j \}) = (v, v(j)), a) 
	\\
	& \;\;\;\;\; - \Pr(\uY(i) | \uX(\Psi) = v, a) \|_1  \geq 
	\\
	& \;\;\;\;\; \|  \Pr(\uY(i) | \uX(\Psi \cup \{ k \}) = (v, v(k)), a)
	\\
	& \;\;\;\;\;\;\;\; - \Pr(\uY(i) | \uX(\Psi) = v, a)  \|_1,
	\end{split}
	\end{equation}
	then:
	\begin{flalign}
	\begin{split}	
	& \| \Pr(\uY(i) | \uX(\Psi \cup \{ j \}) = (v, v(j)), a)
	\\
	& \;\;\;\;\; - \Pr(\uY(i) | \uX(\Psi) = v, a)  \|_1  \geq 
	\\
	& \;\;\;\;\; \| \Pr(\uY(i) | \uX(\Psi \cup \{ j, k \}) = (v, v(j), v(k), a)
	\\
	& \;\;\;\;\;\;\;\; - \Pr(\uY(i) | \uX(\Psi \cup \{ j \}) = (v, v(j)), a)  \|_1	+ C_3 \enspace .\hspace{-1cm}
	\end{split}&
	\end{flalign}
\end{assume}
If conditioning on $\uX(j)$ provides more knowledge on the output distribution than conditioning on another variable $\uX(k)$, then it will also provide more knowledge than conditioning on $\uX(k)$ given $\uX(j)$. In simple words, Assumption \ref{Ass:Submodularity} means that information inferred from variables is monotonic, so influential parents cannot go undetected. This assumption supports our greedy scheme, but there are trivial cases where it does not hold. 

Consider the substructure represented in Figure \ref{fig:ill_Submodularity}: 
\begin{align}
\begin{split}
&\;\;\;\;\;\;\;\;\;\;\;\;\underbrace{
	 	\| \Pr(\var{Y}(3) | \var{X}(1) =i)
	 	- \Pr(\var{Y}(3) ) \|_1
		}_{=0}\not \geq
	 \\
&	\underbrace{
	 	\|  \Pr(\var{Y}(3) | \var{X}(1,2) =(i,j))
	 	- \Pr(\var{Y}(3) | \var{X}(1) =i )\|_1 
		}_{=1}	
.\nonumber
\end{split}
\end{align}
Even though $\var{X}(1,2)$ are together very informative about variable $\var{Y}(3)$, any single one of them is not. In such a situation, useful variables cannot be detected by a greedy scheme. Assumption $\ref{Ass:Submodularity}$ prevents this problem.

\begin{figure}[tbp]
\begin{center}
\def\svgwidth{0.85\columnwidth}
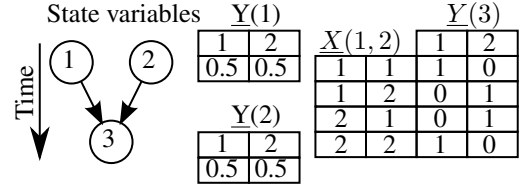
\caption{An FMDP that does not satisfy Assumption \ref{Ass:Submodularity}. See Figure \ref{Ass:Coupling} for an explanation of the representation. }
\label{fig:ill_Submodularity}
\end{center}
\end{figure}

These assumptions form the core hardness of the structure learning problem. From one side, there may be implicit dependencies between variables induced by the dynamics - making it hard to separate non-parents. From the other side, the conditional probabilities may belong to a family of XOR like function - initially hiding attractive true parents. Finally, while these assumptions are crucial for proper analysis, non-parent variables may have a beneficial effect on the actual evaluation error as they still contain information on the true parents values, and subsequently information on the output variable. 

\begin{theorem} \label{thm:Offline Policy Evaluation}
	{	
Suppose Assumptions \ref{Ass:Coupling}, \ref{Ass:dwarfNonParents} and \ref{Ass:Submodularity} hold, and let $\frac{C_1}{4} > \epsilon + \frac{C_2}{4}, \epsilon > 0, \delta_1 > 0$, and $m = \max_{i\in[D]} |\Phi_i|$. Then there exists 
$$
H(\epsilon, \delta_1) = O \left( 
\frac{\Gamma^2}{\delta_1\epsilon^2} \ln \left( \frac{\Gamma}{\delta_1}\right) 
 \right)
$$
such that if \alg\ is given $H$ trajectories, with probability at least $1-2A D(m+2)(D+1-m)  \Gamma^{m+1} \delta_1$, \alg\ returns an evaluation of $\ePi$ satisfying:
	\begin{equation}	
	\left| \nu - \tilde{\nu} \right| \leq %
	T^2 (\delta^* + \epsilon^* D)
	\end{equation}	
	where
	\begin{flalign}
	\begin{split}
	& \epsilon^* = (4m+1) \epsilon + m C_2 + m^2 C_3 
	\\
	& \delta^* =  A \Gamma^m  \sum_{i=1}^D \psi_i \delta_1 
	\\
	& \psi_i = \max\limits_{(v,a) \in F_i} \frac{\sum_{t=1}^T \Pr(\uX_t(\Phi_i)=v, a_t=a | \ePi)}{\sum_{t=1}^T \Pr(\uX_t(\Phi_i)=v, a_t=a | \bPi)} \enspace.\hspace{-0.5cm}
	\end{split}&
	\end{flalign}
}
\end{theorem}

The proof of Theorem \ref{thm:Offline Policy Evaluation} is divided in 4 parts, detailed in the supplementary material. First, we derive a simulation lemma for MDPs stating that for the target policy two MDPs with similar transition probability distributions have proximate value functions. 
We then consider the \emph{number of samples} needed to estimate the transition probabilities of various realization-action pairs. Samples within a trajectory may not be independent so we derive a bound based on Azuma's inequality for martingales.
Subsequently, we consider the \emph{number of trajectories} needed to derive a model that evaluates the target policy accurately. 
If the behavior policy visits enough the parent realizations that the target policy is likely to visit, then the number of trajectories can be small. On the other hand, if the behavior never visits parent realizations that the target policy visits, then the number of trajectories may be infinite. This is captured by $\psi_i$.
Finally, we bound the error due to greedy parent selection under Assumptions \ref{Ass:Coupling}, \ref{Ass:dwarfNonParents} and \ref{Ass:Submodularity}.

The evaluation error bound depends on the horizon $T$, on the number of variables $D$, on the error bound $\epsilon^*$ on most transition probability values of the FMDP constructed by \alg\, and on the probability $T\delta^*$ that a trajectory will not visit a state with badly estimated probability values. The dependency of $\epsilon^*$ on $m$ is the first advantage of the factorization. The constants $C_1$, $C_2$ and $C_3$, from Assumptions \ref{Ass:Coupling}, \ref{Ass:dwarfNonParents} and \ref{Ass:Submodularity}, respectively, indicate the effect of the model ``hardness'' on the bound. When $C_1$ is large enough and $C_2 = C_3 = 0$, the true structure can be learned greedily and the error can be driven arbitrarily close to $0$. In other cases, \alg\ may learn the wrong structure resulting in some approximation error. 

Next, observe the probability that the bounds in Theorem \ref{thm:Offline Policy Evaluation} hold. The multiplicative term $A \Gamma^m$ is unavoidable since for each parents realization and action pair the estimation error on the transition probability must be bounded. The main advantage of this theorem is the lack of a $\Gamma^D$ multiplicative term, which means the effective state space decreased exponentially. The factor $m+2$  is due to the number of iterations of \alg\ where a parent is added, and $D-m+1$ is due to bounds on non-parents that must be valid for all these iterations.

In $\delta^*$, the $\psi_i$ values characterize the mismatch between the behavior policy and the target policy. If the behavior policy visits all of the parent-action realizations that the target policy visits with sufficiently high probability, then the $\psi_i$ parameters will be small. But if the target policy visits parent-action realizations that are never visited by the behavior policy, then the $\psi_i$ values may be infinite. The $\psi_i$ values are similar to importance sampling weights used by some model-free off-policy algorithms. However, unlike model-free approaches that depend on the differences in the state visitation distributions of the behavior policy and the target policy, the $\psi_i$ values depend on the differences in the parent realization visitation distributions between the behavior policy and the target policy. This is more flexible because the $\psi_i$ values can be small even when the behavior policy and the target policy visit different regions of the state-space.

\section{Experiments} \label{Sec:Experiments}

We compared \alg\ to other off-policy evaluation algorithms in the Taxi domain \cite{Dietterich1998}, randomly generated FMDPs, and the Space Invaders domain \cite{Bellemare2013}. Since the domains compared in our experiments have different reward scales, we normalized the errors to compare $\frac{| \nu - \tilde{\nu} |}{| \nu |}$. In all experiments, the behavior policy differs from the target policy. Furthermore, evaluation error always refers to the target policy's evaluation error, and all trajectory data is generated by the behavior policy.

We compare \alg\ to the following algorithms:
\begin{itemize}
\item Model-Free Monte-Carlo (\mfmc, \citealt{Fonteneau2010}): a model-free off-policy evaluation algorithm that constructs artificial target policies by concatenating partial trajectories generated by the behavior policy,
\item Clipped Importance Sampling (\cis, \citealt{Bottou2013}): a model-free importance sampling algorithm that uses a heuristic approach to clip extremely large importance sampling ratios,
\item \flatAlg\ : a flat model-based approach that assumes no structure between any two state-action pairs and simply builds an empirical next state distribution for each state-action pair, and
\item Known Structure (\ks) : a model-based method that is given the true parents, but still needs to estimate the conditional probability tables from data generated by the behavior policy. \ks\ should outperform \alg, because \ks\ knows the structure. We introduce \ks\ to differentiate the evaluation error due to insufficient samples from the evaluation error due to \alg\ selecting the wrong parent variables.
\end{itemize}

Our experimental results show that (1) model-based off-policy evaluation algorithms are more sample efficient than model-free methods, (2) exploiting structure can dramatically improve sample efficiency, and (3) \alg\ often provides a good evaluation of the target policy despite its greedy structure learning approach.

\subsection{Taxi Domain}

The objective in the Taxi domain \cite{Dietterich1998} is for the agent to pickup a passenger from one location and to drop the passenger off at a destination. The state can be described by four variables. We selected the initial state according to a uniform random distribution and used a horizon $T = 200$. The behavior policy selected actions uniform randomly, while the target policy was derived by solving the Taxi domain with the \rmax\ algorithm \cite{Brafman2002}. We discovered that the deterministic policy returned by \rmax\ was problematic for \cis, because the probability of almost all trajectories generated by the behavior policy were 0 with respect to the target policy. To resolve this problem, we modified the policy returned by \rmax\ to ensure that every action is selected in every state with probability at least $\varepsilon = 0.05$.

The Taxi domain is a useful benchmark because we know the true structure and the total number of states is only 500. Thus, we can compare \alg\ to \ks\ and \flatAlg.

\begin{figure}
\centering
\includegraphics[width=0.45\textwidth]{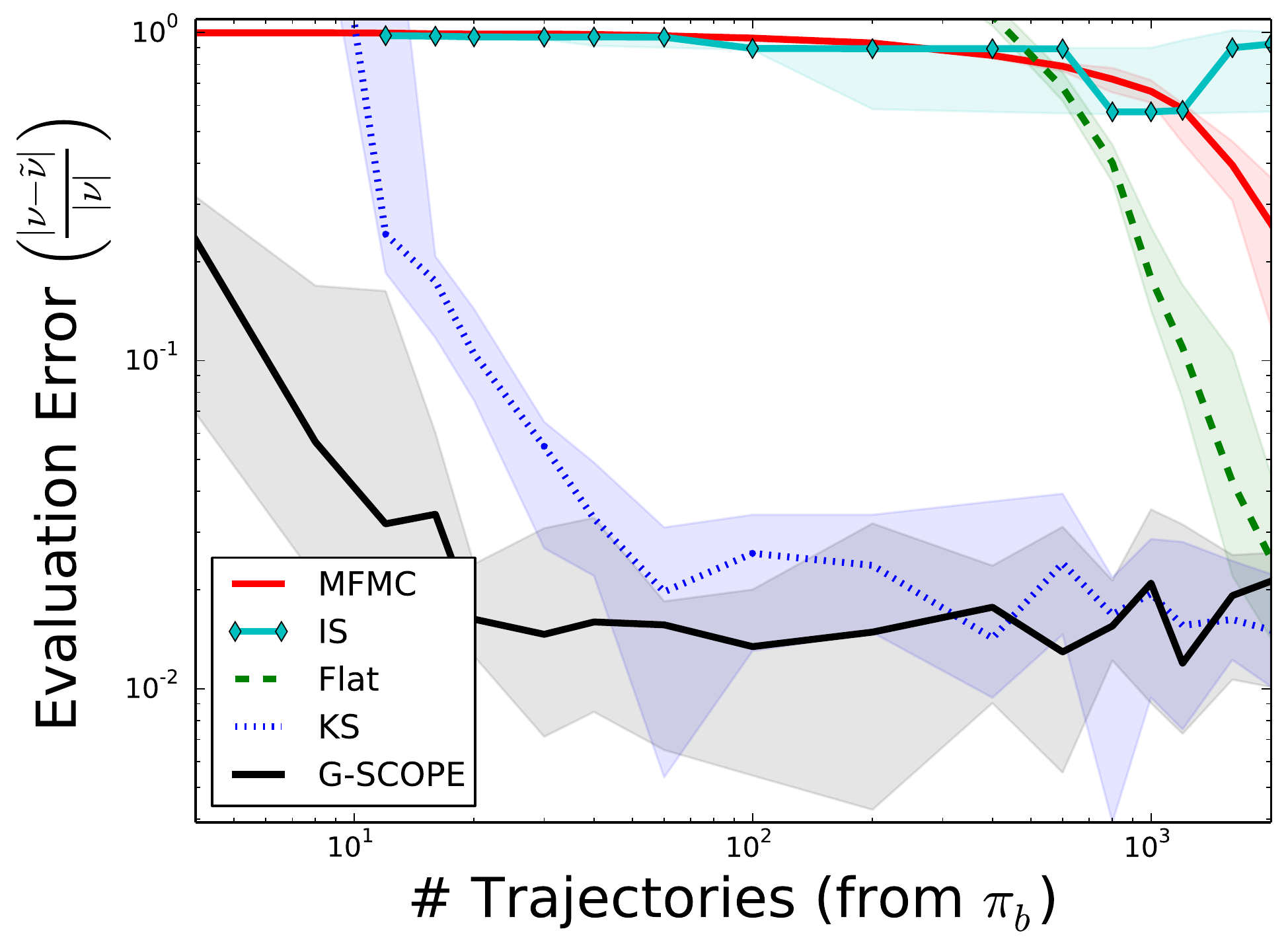}
\caption{Taxi domain: Median evaluation error for the target policy (shaded region: $1^{\rm st} - 3^{\rm rd}$ quantiles) on log-scale for \mfmc, \cis, \flatAlg, \ks, and \alg\ varying the number of trajectories generated by the behavior policy. Without exploiting structure \mfmc\ and \flatAlg\ require many trajectories to achieve small evaluation error. Yet, \ks\ and \alg\ achieve small evaluation error with just a few trajectories. Because \alg\ adapts the complexity of the model to the samples available, it achieve smaller estimation error than even \ks\ for extremely few trajectories. }
\label{fig:taxi_err}
\end{figure}

Figure \ref{fig:taxi_err} presents the normalized evaluation error (on a log-scale) for \mfmc, \cis, \flatAlg, \ks, and \alg\ over 2,000 trajectories generated by the behavior policy. Median and quantiles are estimated over 40 independent trials. For intermediate and large number of trajectories, \alg\ performs about the same as if the structure is given and achieves smaller error than the model-free algorithms (\mfmc\ and \cis). Notice that \mfmc, \cis, and \flatAlg, which do not take advantage of the domains structure, require a large number of trajectories before they achieve low evaluation error. Interestingly, the \flatAlg\ (model-based) approach appears to be more sample efficient than \mfmc, which is in line with observations that model-based RL is more efficient than model-free RL \cite{Hester2009,Jong2007}. \ks\ and \alg, on the other hand, achieve low evaluation error after just a few trajectories and have similar performance, except for very few trajectories where \alg\ can adapt the model complexity to the number of samples and therefore achieves a lower evaluation error than the algorithm knowing the structure. This provides one example where greedy structure learning is effective.

\subsection{Randomly Generated Factored Domains}

To test \alg\ in a higher dimensional problem, where we still know the true structure, we randomly generated FMDPs with $D = 20$ dimensional states. The domain of each variable was $\Gamma = \{ 1, 2 \}$. For each state variable the number of parents was uniformly selected from 1 to 4 and the parents were also chosen randomly. Afterwards, the conditional probability tables were filled in uniformly and normalized to ensure they specified proper probability distributions. The FMDP was given a sparse reward function that returned $1$ if and only if the last bit in the state-vector was $1$ and returned $0$ otherwise. We used a horizon $T = 200$. The behavior policy selected actions uniform randomly, while the target policy was derived by running \sarsa \cite{Sutton1998} with linear value function approximation on the FMDP for 5,000 episodes with a learning rate $0.1$, discount factor $0.9$, and epsilon-greedy parameter $0.05$. After training \sarsa, we extracted a stationary target policy. As in the Taxi domain, we modified the policy returned by \sarsa\ to ensure that every action could be selected in every state with probability at least $\varepsilon = 0.05$.

For the randomly generated FMDPs, we could not construct a flat model because there are $2^{20} = 1,048,576$ states and the number of parameters in a flat model scales quadratically with the size of the state-space. However, we could still compare \mfmc, \cis, \ks, and \alg.

\begin{figure}
\centering
\includegraphics[width=0.5\textwidth]{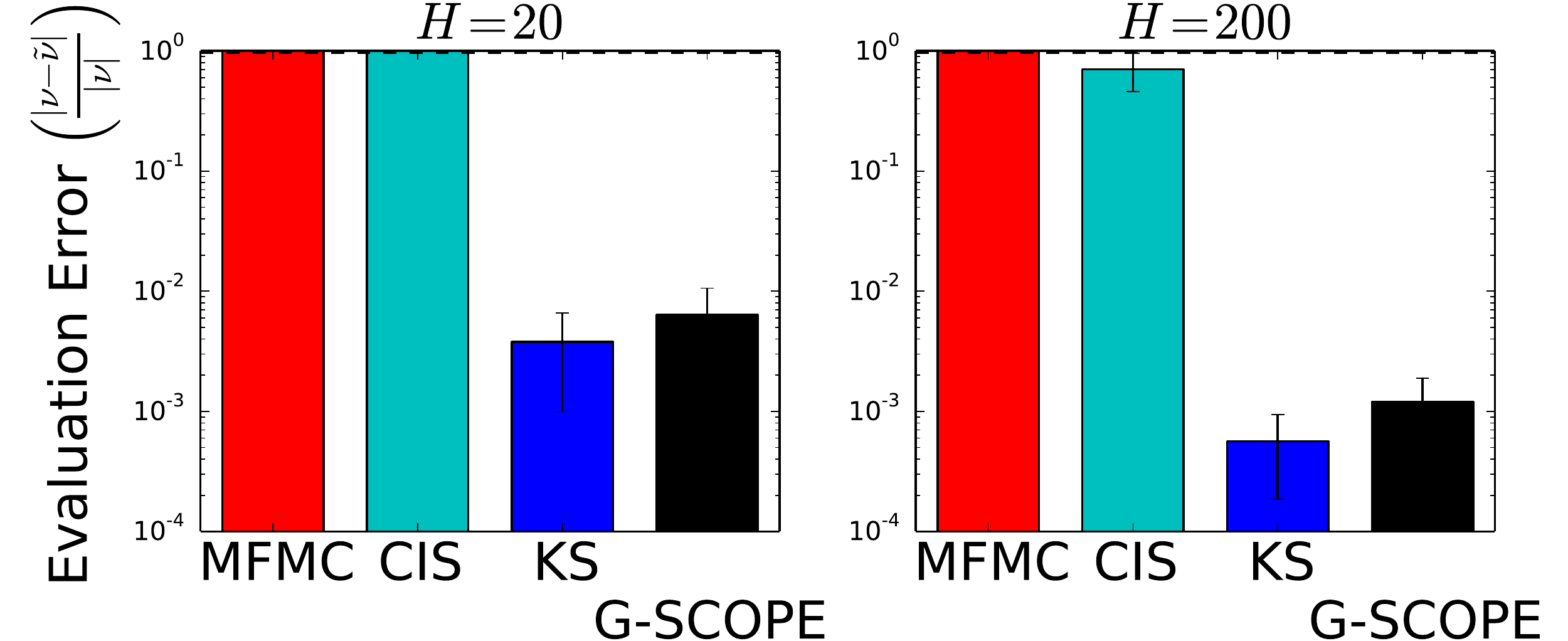}
\caption{Random FMDP domain: Average evaluation error ($\pm 1$ std.\ deviation) on log-scale for \mfmc, \ks, and \alg\ (with $H = 20$ and $200$ trajectories). \alg\ has slightly worse performance than Known Structure, but \alg\ achieves significantly lower evaluation error than \mfmc.}
\label{fig:rfmdp_err}
\end{figure}

Figure \ref{fig:rfmdp_err} presents the normalized evaluation error (on a log-scale) for \mfmc, \cis, \ks, and \alg\ given $H = 20$ and $H = 200$ trajectories from the behavior policy. Average and standard deviations are estimated over 10 independent trials. \mfmc\ fails because in this high-dimensional task there is not enough data to construct artificial trajectories for the target policy. \cis\ fairs only slightly better than \mfmc, because it uses all of the trajectory data. Unfortunately, most of the trajectories generated by the behavior policy are not probable under the target policy and its evaluation of the target policy is pessimistic. \alg\ has slightly worse performance than \ks, but \alg\ achieves significantly lower evaluation error than \mfmc\ and \cis.

\subsection{Space Invaders}

In the Space Invaders (SI) domain using the Arcade Learning Environment \cite{Bellemare2013}, not only do we not know the parent structure, we also cannot verify that the factored dynamics assumption even holds \eqref{Eq:Independency}. Thus, SI presents a challenging benchmark for off-policy evaluation. We used the $1024$-bit RAM as the state vector. We set the horizon $T = 1000$ so that the behavior policy would experience a diverse set of states.

As in the previous experiment, the behavior policy selected actions uniformly at random, while the target policy was derived by running SARSA \cite{Sutton1998} with linear value function approximation on the FMDP with a learning rate $0.1$, discount factor $0.9$, and epsilon-greedy parameter $0.05$. We only trained SARSA for 500 episodes, because of the time required to sample an episode. After training, we extracted a stationary target policy, which ensured all actions could be selected in all states with probability at least $\varepsilon = 0.05$.

\begin{figure}
\centering
\includegraphics[width=0.5\textwidth]{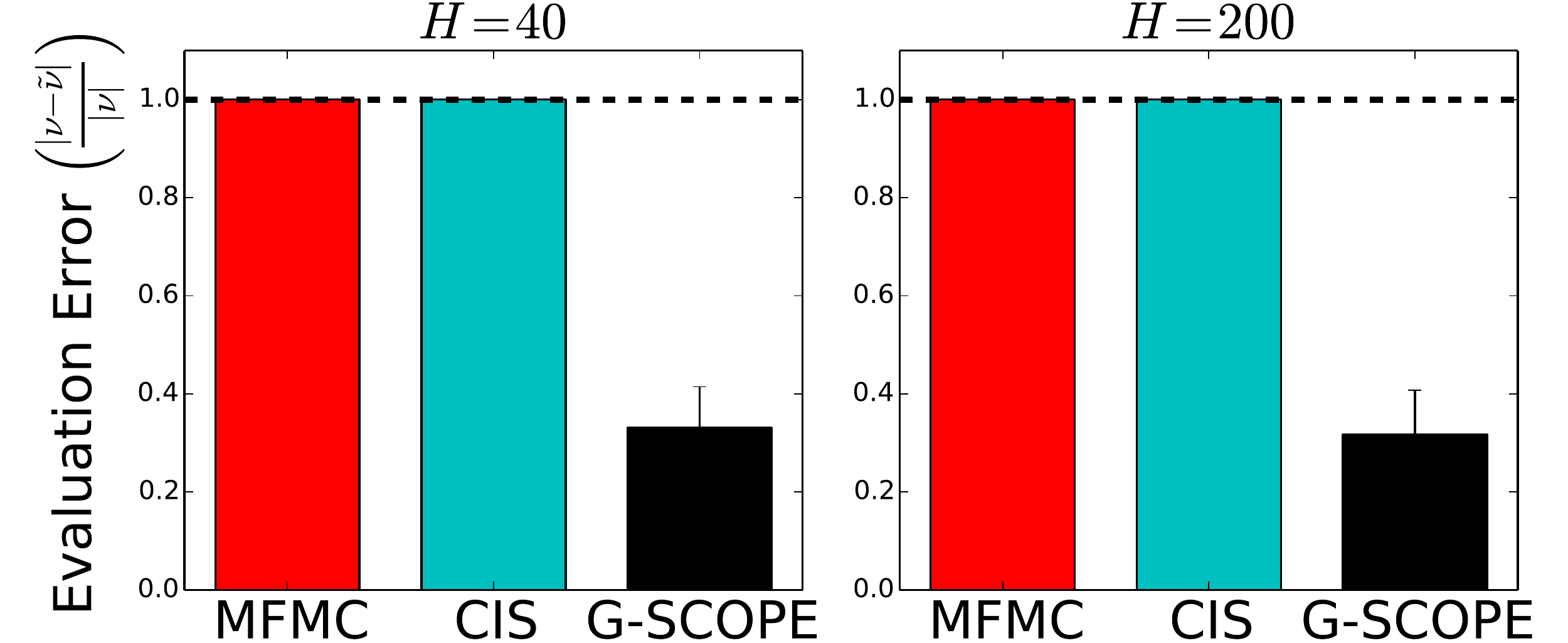}
\caption{Space Invaders domain: Average evaluation error ($\pm 1$ std.\ deviation) for \mfmc, \cis, and \alg\ (with $H = 40$ and $200$ trajectories). \alg\ achieves significantly lower evaluation error than \mfmc\ and \cis.}
\label{fig:si_err}
\end{figure}

Figure \ref{fig:si_err} shows the normalized evaluation error for \mfmc, \cis, and \alg\ given $H = 40$ and $H = 200$ trajectories from the behavior policy. Averages and standard deviations are estimated over 5 independent trials. Again, the evaluation error of \alg\ is much smaller than \mfmc\ and \cis. In fact, \mfmc\ and \cis\ perform no better than a strategy that always predicts the target policy's value is $0$. The poor performance of \mfmc\ is due to the impossibility to construct artificial trajectories from samples in such a high dimensional space.

\section{Discussion} \label{Sec:Discussion}

We presented a finite sample analysis of \alg\ that shows how samples can be related to the evaluation error. When $m \ll D$, the sample complexity scales logarithmically with number of states, where $m = \arg \max_{i \in [D]} |\Phi_i|$.

Our experiments show that (1) model-based off-policy evaluation algorithms are more sample efficient than model-free methods, (2) exploiting structure can dramatically improve sample efficiency, and (3) \alg\ often provides a good evaluation of the target policy despite using a greedy structure learning approach. Thus, \alg\ provides a practical solution for evaluating new policies. Our empirical evaluation on large and small FMDPs shows our approach outperforms existing methods, which only exploit trajectories.

We analyzed \alg\ under three assumptions restricting the class of FMDPs that can be considered. These three assumptions imply that (1) including weak parent will not make any other weak parent (significantly) more informative than it was before, (2) strong parents are more relevant than non-parents, and (3) conditioned on the strong parents non-parents are non-informative. We believe that many real-world problems approximately satisfy these assumptions. If the problem under consideration does not satisfy them, then learning algorithms of combinatorial computational complexity in the number of state variables must be considered to correctly identify the true parents \cite{chakraborty2011structure}.

To the best of our knowledge, this is the first model-based algorithm and analysis for off-policy evaluation in FMDPs. Moreover, \alg\ is a tractable algorithm for learning the structure of an FMDP even if no prior knowledge is given about the order in which variables should be considered. That being said, we hope that showing the effectiveness of structure learning for off-policy evaluation will encourage the adaptation of existing algorithms for learning the structure of FMDPs and more generally dynamic Bayesian networks for off-policy evaluation.

\putbib[OfflineDBNEval]
\end{bibunit}

\begin{bibunit}
\input{appendix}

\putbib[OfflineDBNEval]
\end{bibunit}

\end{document}

%% file: figures/src/svg/Coupling2_not.pdf_tex
\begingroup%
  \makeatletter%
  \providecommand\color[2][]{%
    \errmessage{(Inkscape) Color is used for the text in Inkscape, but the package 'color.sty' is not loaded}%
    \renewcommand\color[2][]{}%
  }%
  \providecommand\transparent[1]{%
    \errmessage{(Inkscape) Transparency is used (non-zero) for the text in Inkscape, but the package 'transparent.sty' is not loaded}%
    \renewcommand\transparent[1]{}%
  }%
  \providecommand\rotatebox[2]{#2}%
  \ifx\svgwidth\undefined%
    \setlength{\unitlength}{192bp}%
    \ifx\svgscale\undefined%
      \relax%
    \else%
      \setlength{\unitlength}{\unitlength * \real{\svgscale}}%
    \fi%
  \else%
    \setlength{\unitlength}{\svgwidth}%
  \fi%
  \global\let\svgwidth\undefined%
  \global\let\svgscale\undefined%
  \makeatother%
  \begin{picture}(1,0.66666667)%
    \put(0,0){\includegraphics[width=\unitlength]{Coupling2_not.pdf}}%
    \put(0.11337587,0.61069485){\makebox(0,0)[lb]{\smash{State variables}}}%
    \put(0.04284302,0.42114735){\rotatebox{90}{\makebox(0,0)[lb]{\smash{Time}}}}%
    \put(0.13468564,0.53064684){\color[rgb]{0,0,0}\makebox(0,0)[b]{\smash{1}}}%
    \put(0.27635237,0.53064684){\color[rgb]{0,0,0}\makebox(0,0)[b]{\smash{2}}}%
    \put(0.13468564,0.39731344){\color[rgb]{0,0,0}\makebox(0,0)[b]{\smash{1}}}%
    \put(0.27635237,0.39731344){\color[rgb]{0,0,0}\makebox(0,0)[b]{\smash{2}}}%
    \put(0.46801901,0.39731344){\color[rgb]{0,0,0}\makebox(0,0)[b]{\smash{3}}}%
    \put(0.46801901,0.28064688){\color[rgb]{0,0,0}\makebox(0,0)[b]{\smash{3}}}%
    \put(0.67072652,0.47878717){\color[rgb]{0,0,0}\makebox(0,0)[b]{\smash{$\var{X}(1,2)$}}}%
    \put(0.89969921,0.53542086){\color[rgb]{0,0,0}\makebox(0,0)[b]{\smash{$\var{Y}(3)$}}}%
    \put(0.62072652,0.42199999){\color[rgb]{0,0,0}\makebox(0,0)[b]{\smash{1}}}%
    \put(0.72315598,0.42199999){\color[rgb]{0,0,0}\makebox(0,0)[b]{\smash{1}}}%
    \put(0.82906011,0.42199999){\color[rgb]{0,0,0}\makebox(0,0)[b]{\smash{1}}}%
    \put(0.93148944,0.42199999){\color[rgb]{0,0,0}\makebox(0,0)[b]{\smash{0}}}%
    \put(0.82906011,0.47498306){\color[rgb]{0,0,0}\makebox(0,0)[b]{\smash{1}}}%
    \put(0.93148944,0.47498306){\color[rgb]{0,0,0}\makebox(0,0)[b]{\smash{2}}}%
    \put(0.62072652,0.36738625){\color[rgb]{0,0,0}\makebox(0,0)[b]{\smash{1}}}%
    \put(0.72315598,0.36738625){\color[rgb]{0,0,0}\makebox(0,0)[b]{\smash{2}}}%
    \put(0.82906011,0.36738625){\color[rgb]{0,0,0}\makebox(0,0)[b]{\smash{0}}}%
    \put(0.93148944,0.36738625){\color[rgb]{0,0,0}\makebox(0,0)[b]{\smash{1}}}%
    \put(0.62072652,0.31424859){\color[rgb]{0,0,0}\makebox(0,0)[b]{\smash{2}}}%
    \put(0.72315598,0.31424859){\color[rgb]{0,0,0}\makebox(0,0)[b]{\smash{1}}}%
    \put(0.82906011,0.31424859){\color[rgb]{0,0,0}\makebox(0,0)[b]{\smash{0}}}%
    \put(0.93148944,0.31424859){\color[rgb]{0,0,0}\makebox(0,0)[b]{\smash{1}}}%
    \put(0.62072652,0.25963491){\color[rgb]{0,0,0}\makebox(0,0)[b]{\smash{2}}}%
    \put(0.72315598,0.25963491){\color[rgb]{0,0,0}\makebox(0,0)[b]{\smash{2}}}%
    \put(0.82906011,0.25963491){\color[rgb]{0,0,0}\makebox(0,0)[b]{\smash{0}}}%
    \put(0.93148944,0.25963491){\color[rgb]{0,0,0}\makebox(0,0)[b]{\smash{1}}}%
    \put(0.31409728,0.12734089){\color[rgb]{0,0,0}\makebox(0,0)[b]{\smash{$\var{X}(1)$}}}%
    \put(0.48954401,0.18397458){\color[rgb]{0,0,0}\makebox(0,0)[b]{\smash{$\var{Y}(1)$}}}%
    \put(0.31300077,0.07055372){\color[rgb]{0,0,0}\makebox(0,0)[b]{\smash{1}}}%
    \put(0.41890488,0.07055372){\color[rgb]{0,0,0}\makebox(0,0)[b]{\smash{1}}}%
    \put(0.52133427,0.07055372){\color[rgb]{0,0,0}\makebox(0,0)[b]{\smash{0}}}%
    \put(0.41890488,0.12353675){\color[rgb]{0,0,0}\makebox(0,0)[b]{\smash{1}}}%
    \put(0.52133427,0.12353675){\color[rgb]{0,0,0}\makebox(0,0)[b]{\smash{2}}}%
    \put(0.31300077,0.01593997){\color[rgb]{0,0,0}\makebox(0,0)[b]{\smash{2}}}%
    \put(0.41890488,0.01593997){\color[rgb]{0,0,0}\makebox(0,0)[b]{\smash{0}}}%
    \put(0.52133427,0.01593997){\color[rgb]{0,0,0}\makebox(0,0)[b]{\smash{1}}}%
    \put(0.72475004,0.12438882){\color[rgb]{0,0,0}\makebox(0,0)[b]{\smash{$\var{X}(2)$}}}%
    \put(0.90019684,0.18397455){\color[rgb]{0,0,0}\makebox(0,0)[b]{\smash{$\var{Y}(2)$}}}%
    \put(0.72365354,0.07055369){\color[rgb]{0,0,0}\makebox(0,0)[b]{\smash{1}}}%
    \put(0.82955767,0.07055369){\color[rgb]{0,0,0}\makebox(0,0)[b]{\smash{1}}}%
    \put(0.93198713,0.07055369){\color[rgb]{0,0,0}\makebox(0,0)[b]{\smash{0}}}%
    \put(0.82955767,0.12353671){\color[rgb]{0,0,0}\makebox(0,0)[b]{\smash{1}}}%
    \put(0.93198713,0.12353671){\color[rgb]{0,0,0}\makebox(0,0)[b]{\smash{2}}}%
    \put(0.72365354,0.01593994){\color[rgb]{0,0,0}\makebox(0,0)[b]{\smash{2}}}%
    \put(0.82955767,0.01593994){\color[rgb]{0,0,0}\makebox(0,0)[b]{\smash{0}}}%
    \put(0.93198713,0.01593994){\color[rgb]{0,0,0}\makebox(0,0)[b]{\smash{1}}}%
    \put(0.12627691,0.30599149){\color[rgb]{0,0,0}\makebox(0,0)[b]{\smash{$\var{Y}(1)$}}}%
    \put(0.0556378,0.19257062){\color[rgb]{0,0,0}\makebox(0,0)[b]{\smash{0.5}}}%
    \put(0.15806719,0.19257062){\color[rgb]{0,0,0}\makebox(0,0)[b]{\smash{0.5}}}%
    \put(0.0556378,0.24555365){\color[rgb]{0,0,0}\makebox(0,0)[b]{\smash{1}}}%
    \put(0.15806719,0.24555365){\color[rgb]{0,0,0}\makebox(0,0)[b]{\smash{2}}}%
    \put(0.12627691,0.13099149){\color[rgb]{0,0,0}\makebox(0,0)[b]{\smash{$\var{Y}(2)$}}}%
    \put(0.05563775,0.01757062){\color[rgb]{0,0,0}\makebox(0,0)[b]{\smash{0.5}}}%
    \put(0.1580672,0.01757062){\color[rgb]{0,0,0}\makebox(0,0)[b]{\smash{0.5}}}%
    \put(0.05563775,0.07055365){\color[rgb]{0,0,0}\makebox(0,0)[b]{\smash{1}}}%
    \put(0.1580672,0.07055365){\color[rgb]{0,0,0}\makebox(0,0)[b]{\smash{2}}}%
  \end{picture}%
\endgroup%

%% file: figures/src/svg/Submodularity_not.pdf_tex
\begingroup%
  \makeatletter%
  \providecommand\color[2][]{%
    \errmessage{(Inkscape) Color is used for the text in Inkscape, but the package 'color.sty' is not loaded}%
    \renewcommand\color[2][]{}%
  }%
  \providecommand\transparent[1]{%
    \errmessage{(Inkscape) Transparency is used (non-zero) for the text in Inkscape, but the package 'transparent.sty' is not loaded}%
    \renewcommand\transparent[1]{}%
  }%
  \providecommand\rotatebox[2]{#2}%
  \ifx\svgwidth\undefined%
    \setlength{\unitlength}{204bp}%
    \ifx\svgscale\undefined%
      \relax%
    \else%
      \setlength{\unitlength}{\unitlength * \real{\svgscale}}%
    \fi%
  \else%
    \setlength{\unitlength}{\svgwidth}%
  \fi%
  \global\let\svgwidth\undefined%
  \global\let\svgscale\undefined%
  \makeatother%
  \begin{picture}(1,0.37254902)%
    \put(0,0){\includegraphics[width=\unitlength]{Submodularity_not.pdf}}%
    \put(0.06749101,0.32379128){\makebox(0,0)[lb]{\smash{State variables}}}%
    \put(0.04032285,0.12970741){\rotatebox{90}{\makebox(0,0)[lb]{\smash{Time}}}}%
    \put(0.11107669,0.23276572){\color[rgb]{0,0,0}\makebox(0,0)[b]{\smash{1}}}%
    \put(0.26793947,0.23276572){\color[rgb]{0,0,0}\makebox(0,0)[b]{\smash{2}}}%
    \put(0.18950818,0.07590291){\color[rgb]{0,0,0}\makebox(0,0)[b]{\smash{3}}}%
    \put(0.90952083,0.32353335){\color[rgb]{0,0,0}\makebox(0,0)[b]{\smash{$\var{Y}(3)$}}}%
    \put(0.68617398,0.27023106){\color[rgb]{0,0,0}\makebox(0,0)[b]{\smash{$\var{X}(1,2)$}}}%
    \put(0.64695828,0.21678428){\color[rgb]{0,0,0}\makebox(0,0)[b]{\smash{1}}}%
    \put(0.74336249,0.21678428){\color[rgb]{0,0,0}\makebox(0,0)[b]{\smash{1}}}%
    \put(0.84303694,0.21678428){\color[rgb]{0,0,0}\makebox(0,0)[b]{\smash{1}}}%
    \put(0.93944107,0.21678428){\color[rgb]{0,0,0}\makebox(0,0)[b]{\smash{0}}}%
    \put(0.84303694,0.26665067){\color[rgb]{0,0,0}\makebox(0,0)[b]{\smash{1}}}%
    \put(0.93944107,0.26665067){\color[rgb]{0,0,0}\makebox(0,0)[b]{\smash{2}}}%
    \put(0.64695828,0.16538311){\color[rgb]{0,0,0}\makebox(0,0)[b]{\smash{1}}}%
    \put(0.74336249,0.16538311){\color[rgb]{0,0,0}\makebox(0,0)[b]{\smash{2}}}%
    \put(0.84303694,0.16538311){\color[rgb]{0,0,0}\makebox(0,0)[b]{\smash{0}}}%
    \put(0.93944107,0.16538311){\color[rgb]{0,0,0}\makebox(0,0)[b]{\smash{1}}}%
    \put(0.64695828,0.11537123){\color[rgb]{0,0,0}\makebox(0,0)[b]{\smash{2}}}%
    \put(0.74336249,0.11537123){\color[rgb]{0,0,0}\makebox(0,0)[b]{\smash{1}}}%
    \put(0.84303694,0.11537123){\color[rgb]{0,0,0}\makebox(0,0)[b]{\smash{0}}}%
    \put(0.93944107,0.11537123){\color[rgb]{0,0,0}\makebox(0,0)[b]{\smash{1}}}%
    \put(0.64695828,0.0639701){\color[rgb]{0,0,0}\makebox(0,0)[b]{\smash{2}}}%
    \put(0.74336249,0.0639701){\color[rgb]{0,0,0}\makebox(0,0)[b]{\smash{2}}}%
    \put(0.84303694,0.0639701){\color[rgb]{0,0,0}\makebox(0,0)[b]{\smash{1}}}%
    \put(0.93944107,0.0639701){\color[rgb]{0,0,0}\makebox(0,0)[b]{\smash{0}}}%
    \put(0.47963317,0.32349574){\color[rgb]{0,0,0}\makebox(0,0)[b]{\smash{\var{Y}(1)}}}%
    \put(0.41314931,0.21674667){\color[rgb]{0,0,0}\makebox(0,0)[b]{\smash{0.5}}}%
    \put(0.50955344,0.21674667){\color[rgb]{0,0,0}\makebox(0,0)[b]{\smash{0.5}}}%
    \put(0.41314931,0.26661307){\color[rgb]{0,0,0}\makebox(0,0)[b]{\smash{1}}}%
    \put(0.50955344,0.26661307){\color[rgb]{0,0,0}\makebox(0,0)[b]{\smash{2}}}%
    \put(0.47963317,0.12328612){\color[rgb]{0,0,0}\makebox(0,0)[b]{\smash{\var{Y}(2)}}}%
    \put(0.41314925,0.01653706){\color[rgb]{0,0,0}\makebox(0,0)[b]{\smash{0.5}}}%
    \put(0.50955344,0.01653706){\color[rgb]{0,0,0}\makebox(0,0)[b]{\smash{0.5}}}%
    \put(0.41314925,0.06640345){\color[rgb]{0,0,0}\makebox(0,0)[b]{\smash{1}}}%
    \put(0.50955344,0.06640345){\color[rgb]{0,0,0}\makebox(0,0)[b]{\smash{2}}}%
  \end{picture}%
\endgroup%

%% file: appendix.tex
\newpage
\appendix
\onecolumn

\section{List of Notations} \label{app:notations}
\begin{tabular}{|l|l|}
	\hline
	Notation & Meaning \\ \hline
	$A$ & Action space \\ \hline	
	$T$ & Time horizon \\ \hline
	$t$ & Time index $t=0..T$ \\ \hline
	$H$ & Number of trajectories in batch data \\ \hline
	$D$ & Number of factors in each state \\ \hline
	$[D]$ & The set $1,2,..,D$. \\ \hline
	$\Gamma$ & Domain of each factor in a state and (dual notation) the number of possible values for the factor \\ \hline
	$M$ & Markov Decision Process \\ \hline
	$\rho$ & Distribution of first state in MDP \\ \hline
	$\uX$ & Input variable (represents previous state) \\ \hline
	$\uY$ & Output variable (represents next state) \\ \hline
	$\uY(i)$ & The $i$'th variable in the output. \\ \hline
	$\Psi$ & A subset of indices \\ \hline
	$\uX(\Psi)$ & The corresponding subset of variables to $\Psi$ \\ \hline
	$\Phi_i$ & Indices of the parents for variable $i$ \\ \hline
	$m$ & $\max_i  |\Phi_i|$ \\ \hline
	$F_i$ & $\Gamma^{ |\Phi_i| } \times A$, the set of all realization-action pairs for the parents of node $i$ \\ \hline
	$\hat{\Phi}_i$ & Indices found by \alg\ for variable $i$ \\ \hline
	$n(instance)$ & Number of observations in the data fitting the instance \\ \hline
	$\Theta_i$ & The set of realization-action pairs observed more than $N(\epsilon, \delta)$ for each $\hat{\Phi}_i$ \\ \hline
	$\psi_i$ & A value signifying policies mismatch (bigger means higher mismatch) \\ \hline
\end{tabular}

\section{Proof of Main Theorem \& Supporting Lemmas} \label{app:appendix}

The proof of Theorem \ref{thm:Offline Policy Evaluation} is broken down into parts.

\input{sim_lemma}

\input{l1bound}

\subsection{Bounding the Number of Trajectories}

In this subsection, we derive a bound on the number of trajectories needed to derive a model that evaluates the target policy accurately. Notice that the learned model does not need to be accurate everywhere -- only the regions of the state space where the target policy is likely to visit (and in a FMDP only the parent realizations that the target policy is likely to visit). Our analysis takes advantage of this. When the behavior policy visits the parent realizations that the target policy is likely to visit, then the number of trajectories can be small. On the other hand, if the behavior policy never visits parent realizations that the target policy visits, then the number of trajectories may be infinite.

We will make use of the following Proposition proved in \citet{Li2009}.

\begin{prop} \label{prop:li_lemma} \citep[Lemma 56]{Li2009}
Let $k \in \mathbb{N}$, $\mu, \delta \geq (0, 1)$, $B_1, B_2, \dots B_m$ be a sequence of $m$ independent Bernoulli random variables such that $\mathbb{E}\left[ B_i \right] \geq \mu$ for $i = 1, 2, \dots , m$, and
\begin{equation}
m \geq \frac{2}{\mu} \left( k + \ln \frac{1}{\delta} \right) \enspace ,
\end{equation}
then
\begin{equation}
\Pr \left[ \sum\limits_{i=1}^{m} B_i \geq k \right] \geq 1 - \delta \enspace .
\end{equation}
\end{prop}

Proposition \ref{prop:li_lemma} tells us the number of experiments we need to perform on a Bernoulli distribution to observe at least $k$ successes with high probability. The following corollary modifies the statement of Proposition \ref{prop:li_lemma} to tell us the number of experiments we need to perform to see a high probability set of outcomes from a categorical distribution at least $k$ times with high probability.

\begin{cor} \label{cor:li_lemma} (to Proposition \ref{prop:li_lemma})
Let $\delta \in (0, 1]$, $k \geq 1$, $\Gamma$ be a finite set, $\rho \in \mathcal{M}(\Gamma)$ be a probability distribution with outcomes from $\Gamma$, and $X_1, X_2, \dots, X_m$ be independent random variables sampled from $\rho$. Let $S_m^{k} = \left\{ x \in \Gamma \mid \sum_{i=1}^{m} \mathbb{I}\{ X_i = x \} \geq k \right\}$ be the set of elements encountered $k$ or more times and $\bar{S}_m^{k} = \Gamma \backslash S_m^{k}$ be its complement. If 
\begin{equation}
m \geq \frac{2|\Gamma|}{\delta} \left( k + \ln \frac{|\Gamma|}{\delta} \right) \enspace ,
\end{equation}
then, with probability at least $1-\delta$,
\begin{equation} \label{eqn:observed_prob}
\Pr_{x \sim \rho} \left[ x \in \bar{S}_m^{k} \right] < \delta \enspace ,
\end{equation}
the set of outcomes visited less than $k$ times has total probability mass less than $\delta$.
\end{cor}
\begin{proof}
Consider an infinite sequence of random variables $X_1, X_2, \dots $ distributed according to $\rho$. Denote by $j[1] < j[2] < \dots < j[k]$ the indices resulting in the event that $X_{i} \in S_{i}^{k}$ and $X_i \notin S_{i-1}^{k}$. Notice that we let an index $j[l]$ be infinite in the case that the event never occurs. However, $\Gamma$ only contains $|\Gamma|$ elements, so an element can be added to $S_{\cdot}^{k}$ at most $|\Gamma|$ times. Notice that $S_{j[l]}^{k} = S_{j[l] + 1}^{k} = ... = S_{j[l+1]-1}^{k}$ for $l=1,2, \dots, |\Gamma|-1$. We construct Bernoulli random variables
\begin{equation} \label{eqn:success}
B_i = \left\{ \begin{array}{ll} 1 & \text{if } X_i \notin S_{i-1}^{k} \\ 0 & \text{otherwise} \end{array} \right.
\end{equation}
for $i \geq 1$. So $B_{j[l]+1}, B_{j[l]+2}, \dots , B_{j[l+1]}$ are independent, identically distributed Bernoulli random variables for $l = 1, 2, \dots, |\Gamma|-1$ (but $B_{j[l]}$ and $B_{j[l+1]}$ are not independent). Suppose that $\Pr[B_{j[l]}=1] \geq \delta$ for some $l \in \{ 1, 2, 3, \dots , |\Gamma| \}$ (this is at least true for $\Pr[B_{j[1]} = 1] = 1 \geq \delta$), then by Proposition \ref{prop:li_lemma}, (with $\mu \leftarrow \delta, \delta \leftarrow \frac{\delta}{|\Gamma|}$)
$$
 j[l+1] - (j[l]+1)  \leq \frac{2}{\delta} \left( k + \ln \frac{|\Gamma|}{\delta} \right) \enspace ,
$$ 
with probability at least $1-\frac{\delta}{|\Gamma|}$. Since there are only $|\Gamma|$ elements in $\Gamma$, this can only happen at most $|\Gamma|$ times. Thus, by the union bound, after $m \geq \frac{2|\Gamma|}{\delta} \left( k + \ln \frac{|\Gamma|}{\delta} \right)$ samples, either all $|\Gamma|$ outcomes have been observed or $\Pr[B_m = 1] < \delta$, with probability at least $1 - |\Gamma| \frac{\delta}{|\Gamma|} = 1 - \delta$. If we have observed all $|\Gamma|$ elements then \eqref{eqn:observed_prob} holds trivially. On the other hand if $\Pr[B_m = 1] < \delta$, then 
\begin{align*}
\delta >& \Pr_{x \sim \rho} [ x \notin S_{m-1}^{k} ] \enspace , && \text{By the definition of $B_m$ \eqref{eqn:success}.} \\
\geq& \Pr_{x \sim \rho} [ x \notin S_{m}^{k} ] , && \text{The probability of a success} \\
& && \text{decreases because $S_{m-1}^{k} \subseteq S_{m}^{k}$.} \\
=& \Pr_{x \sim \rho} [ x \in \bar{S}_{m}^{k} ] \enspace .
\end{align*} 
\end{proof}

\begin{prop} \label{prop:mass_cover}
Let $\delta \in (0, 1]$, $k \geq 1$, $\Gamma$ be a finite set, $\rho, \mu \in \mathcal{M}(\Gamma)$ be probability distributions with outcomes from $\Gamma$, and $X_1, X_2, \dots, X_n$ be independent random variables sampled from $\rho$. Let $S_n^{k} = \left\{ x \in \Gamma \mid \sum_{i=1}^{n} \mathbb{I}\{ X_i = x \} \geq k \right\}$ be the set of elements encountered $k$ or more times and $\bar{S}_n^{k} = \Gamma \backslash S_n^{k}$ be its complement. If 
\begin{equation}
n\geq \frac{2|\Gamma|}{\delta} \left( k + \ln \frac{|\Gamma|}{\delta} \right) \enspace ,
\end{equation}
then, with probability at least $1-\delta$,
$$
\Pr_{x \sim \mu} \left[ x \in \bar{S}_n^{k} \right] < \psi \delta \enspace ,
$$
where $\psi = \max\limits_{x \in \Gamma} \frac{\mu(x)}{\rho(x)} $ (taking $\frac{0}{0} = 0$).
\end{prop}
\begin{proof}

We want to show $\Pr_{x \sim \mu}\left[ x \in \bar{S}_n^k \right] < \psi \delta$. By applying Corollary \ref{cor:li_lemma}, we have that $\Pr_{x \sim \rho}\left[ x \in \bar{S}_n^k \right] < \delta$ with probability at least $1-\delta$. It suffices to show that $\Pr_{x \sim \mu} \left[ x \in \bar{S}_n^k \right] \leq \psi \Pr_{x \sim \rho} \left[ x \in \bar{S}_n^k \right] \leq \psi \delta$.

\begin{align*}
\Pr_{x \sim \mu} \left[ x \in \bar{S}_{n}^{k} \right] &= \sum\limits_{x \in \bar{S}_{n}^{k}} \mu(x) \\
&= \sum\limits_{x \in \bar{S}_{n}^{k}} \mu(x) \frac{\rho(x)}{\rho(x)} \\
&= \sum\limits_{x \in \bar{S}_{n}^{k}} \rho(x) \frac{\mu(x)}{\rho(x)} \\
&\leq \left(\max_{y \in \Gamma} \frac{\mu(y)}{\rho(y)} \right) \sum\limits_{x \in \bar{S}_{n}^{k}} \rho(x) \\
&= \psi \Pr_{x \sim \rho}\left[ x \notin S_n \right] \enspace .
\end{align*}

\end{proof}

For completeness we introduce the following proposition that is used to prove our lemma.

\begin{prop} \label{prop:BoundingFactoredDeviations}
\cite{Osband:2014aa}
Let $\var{Y}(i)$ be a set of variables indexed by $i\in[D]$, $v_i$ a realization of $\var{Y}(i)$, $v=(v_1,...,v_D)$ and $\Pr_1,\Pr_2$ be two factorized probability distributions over $\var{Y}$:
\begin{align}
\Pr_j(\var{Y})=\prod_{i=1}^D \Pr_j(\var{Y}(i)) \qquad j=1,2\enspace.
\end{align}
Then
\begin{align}
||\Pr_1(\var{Y}=v)-\Pr_2(\var{Y}=v)||_{1} \leq \sum_{i=1}^D ||\Pr_1(\var{Y}(i)=v_i) -\Pr_2(\var{Y}(i)=v_i)||_1\enspace.
\end{align}
\end{prop}

\begin{lemma} \label{lem:num_trajectories}
	Let $\epsilon, \delta >0$. If the number of trajectories
$$
H \geq \frac{4AD\Gamma^m}{\delta} \left( \frac{2\Gamma^2}{\epsilon^2} \ln \left( \frac{4AD\Gamma^{m+1}}{\delta}\right) + \ln \left( \frac{2AD\Gamma^{m}}{\delta} \right) \right) \enspace ,
$$ 
then, with probability at least $1-\delta$, there is a subset of state-action pairs 
$$
K = \left\{ (\uX,a) \in S \times A \mid \| \Pr(\uY|\uX,a) - \widehat{\Pr}(\uY|\uX,a) \|_1 \leq D\epsilon \right\} \enspace ,
$$
such that:
	\begin{equation} 
	\Pr\left[ \exists_{t \in [T]} (\uX_t,a_t) \notin K \mid M, \ePi \right] < \frac{ T \sum_{i=1}^D \psi_i \delta }{2D} 
	\end{equation}
	where $\psi_i = \max\limits_{(v,a) \in F_i} \frac{\sum_{t=1}^T \Pr(X_t(\Phi_i)=v, a_t=a | \ePi)}{\sum_{t=1}^T \Pr(X_t(\Phi_i)=v, a_t=a | \bPi)}$.

\end{lemma}

\newcommand{\uW}{\underline{W}} %

\begin{proof}
For every $i\in [D]$, we define the random variable $\uW$: 
\begin{quote}
For a given trajectory sample a time $t$ uniformly and set $\uW = (X_t(\Phi_i), a_t)$.
\end{quote}
Notice that $\uW$ is distributed according to the distribution induced by the behavior policy $\bPi$ and that $\uW$ receives one of $A \Gamma^{| \Phi_i |}$ values. We denote the distribution of $\uW$ by $\rho$ and over the target policy by $\mu$. Setting $k = \frac{2\Gamma^2}{\epsilon^2} \ln \left( \frac{2\Gamma}{\delta_1} \right)$ and using Proposition \ref{prop:mass_cover} we obtain that having:
\begin{equation}
H \geq \frac{2 A  \Gamma ^{| \Phi_i |} }{\delta_2} \left(\frac{2\Gamma^2}{\epsilon^2} \ln \left( \frac{2\Gamma}{\delta_1} \right)  + \ln \frac{ A \Gamma^{| \Phi_i |} }{\delta_2} \right) \enspace ,
\end{equation}
samples from $\rho$, with probability at least $1-\delta_2$,
$$
\Pr_{(v,a) \sim \mu} \left[ (v,a) : n(v,a) \leq \frac{2\Gamma^2}{\epsilon^2} \ln \left( \frac{2|\Gamma|}{\delta_1} \right) \right] < \psi_i \delta_2 \enspace ,
$$
where $\psi_i = \max\limits_{(v,a) \in F_i} \frac{\mu(v,a)}{\rho(v,a)} = \frac{\sum_{t=1}^T \Pr(X_t(\Phi_i)=v, a_t=a | \ePi)}{\sum_{t=1}^T \Pr(X_t(\Phi_i)=v, a_t=a | \bPi)}$ (taking $\frac{0}{0} = 0$).

By Lemma \ref{Lem:Sample Complexity} and given our choice for $k\equiv N(\epsilon,\delta_1)$, if we have observed $N(\epsilon,\delta_1)$ samples from $\Pr(\uY(i)|\uX(\Phi_i)=v,a)$, then our estimate $\widehat{\Pr}$ satisfies
$$
\left\| \Pr(\uY(i)|(\uX(\Phi_i)=v,a)) - \widehat{\Pr}(\uY(i)|(\uX(\Phi_i)=v,a)) \right\|_1 \leq \epsilon \enspace ,
$$
with probability at least $1-\delta_1$. Now denote by 
$$
K_i = \left\{ (v,a) \in F_i \mid \left\| \Pr(\uY(i)|(\uX(\Phi_i)=v,a)) - \widehat{\Pr}(\uY(i)|(\uX(\Phi_i)=v,a)) \right\|_1 \leq \epsilon\right\} \enspace ,
$$
the set of realization-action pairs for predicting the $i^{\rm th}$ output variable where the empirical distribution estimated from trajectory data is $\epsilon$-close to the true distribution.

By applying the union bound over at most $A\Gamma^{\Phi_i}$ realization-action pairs, after $H$ trajectories, we have 
$$
	\Pr\left[ \exists_{t \in [T]} (\uX_t(\Phi_i),a_t) \notin K_i \mid M, \ePi \right] \leq T\psi_i\delta_2 \enspace ,
$$
with probability at least $1-(\delta_2 + \delta_1 A \Gamma^{|\Phi_i|})$. By applying the union bound again over all $D$ output variables, we obtain
$$
	\sum_{i=1}^{D} \Pr\left[ \exists_{t \in [T]} (\uX_t(\Phi_i),a_t) \notin K_i \mid M, \ePi \right] \leq T\sum_{i=1}^{D} \psi_i\delta_2
$$
with probability at least $1-D(\delta_2 + A\Gamma^{|\Phi_i|} \delta_1)$. Notice that this implies
\begin{align*}
\Pr\left[ \exists_{t \in [T]} (\uX_t, a_t) \notin K \mid M, \ePi \right] &\leq \sum_{i=1}^{D} \Pr\left[ \exists_{t \in [T]} (\uX_t(\Phi_i),a_t) \notin K_i \mid M, \ePi \right] \\
&\leq T\sum_{i=1}^{D} \psi_i\delta_2 \enspace ,
\end{align*}
holds with probability at least $1-D(\delta_2 + A\Gamma^{|\Phi_i|} \delta_1) \geq 1-D(\delta_2 + A \Gamma^{m} \delta_1)$.

The bound over $\|\Pr(\uY|\uX,a) - \widehat{\Pr}(\uY|\uX,a)\|_1$ directly results from Proposition \ref{prop:BoundingFactoredDeviations}.

Setting:
\begin{equation} \label{eq:delta_1}
\begin{split}
& \frac{\delta}{2} = D \delta_2 \enspace \Rightarrow \delta_2 = \frac{\delta}{2D}\\ 
& \frac{\delta}{2} = \delta_1 A D |\Gamma|^m \enspace \Rightarrow \delta_1 = \frac{\delta}{2 A D \Gamma^m}
\end{split}
\end{equation}
We can rewrite $H$ in terms of $\epsilon, \delta$:
\begin{equation}
H \geq \frac{4AD\Gamma^m}{\delta} \left( \frac{2\Gamma^2}{\epsilon^2} \ln \left( \frac{4AD\Gamma^{m+1}}{\delta}\right) + \ln \left( \frac{2AD\Gamma^{m}}{\delta} \right) \right)
\end{equation}

\end{proof}

\subsection{Error due to Greedy Parent Selection}
\begin{lemma} \label{lem:model_error}
	Suppose Assumptions \ref{Ass:Coupling}, \ref{Ass:dwarfNonParents} and \ref{Ass:Submodularity}  hold. Let $\epsilon > 0, \delta_1 > 0$, and
$$
\frac{C_1}{4} > \epsilon + \frac{C_2}{4} \enspace .
$$ 
After applying \alg, for every $i\in [D]$, $(v,a)\in \Theta_i$, and every $w \in \Gamma^{ | \Phi_i |}$ satisfying $w(\hat{\Phi}_i) = v(\Phi_i), N(w, a) \geq N(\epsilon, \delta_1) $:
	\begin{equation}
	\| \Pr(\uY(i) | \uX(\Phi_i) = w, a) - \widehat{\Pr}(\uY(i) | \uX(\hat{\Phi}_i) = v, a) \|_1 \leq  (4D+1) \epsilon + D^2 C_3 
	\end{equation}
	with probability at least $1- 2D (m+1)(D+1-m)A\Gamma^{m+1}  \delta_1$.
\end{lemma}

\begin{proof}

This Lemma is only concerned about realization-action pairs for which there are enough samples. \alg\ will not consider the score of realization-action pairs that do not have enough sample. When constructing the structure, this automatically discard realization-action pairs containing non-parents that do not meet the  number of samples required to have an estimation error bounded by $\epsilon$ with high probability. Hence, in what follows, we will always consider the worse case where there are always enough samples to estimate such probabilities.

To simplify notation, let
\begin{align}
\widehat{\alpha}(k,v,v_k,a) &= \| \widehat{\Pr}(\uY(i) | \uX(\hat{\Phi}_i )=v, a ) - \widehat{\Pr}(\uY(i) | \uX(\hat{\Phi}_i \cup \{ k \}) = (v,v_k), a ) \|_1 \\
\alpha(k,v,v_k,a) &= \| \Pr(\uY(i) | \uX(\hat{\Phi}_i )=v, a ) - \Pr(\uY(i) | \uX(\hat{\Phi}_i \cup \{ k \}) = (v,v_k), a ) \|_1 \\
\widehat{\alpha}^*(k) &= \max_{v,v_k,a}\widehat{\alpha}(k,v,v_k,a) \\
\alpha^*(k) &= \max_{v,v_k,a}{\alpha}(k,v,v_k,a) \\
(v^*,v^*_k,a^*) &= \max_{v,v_k,a}{\alpha}(k,v,v_k,a)
\enspace .
\end{align}

We want to bound the probability that \alg\ adds any non-parent variable. The \alg\ algorithm can only select a variable $k$ to add to the parent set only if the following necessary condition holds:
$$
\widehat{\alpha}^*(k) \geq\max_{j \in D \backslash \hat{\Phi}_i} \widehat{\alpha}^*(j) \enspace .
$$

We break up this first part of the proof into two distinct, successive cases.
\begin{enumerate}
\item $\exists {k \in \Phi_i^s}$ that is not in $\hat{\Phi}_i$ (\alg\ has not added all of the strong parents yet), and
\item $\Phi_i^s \subseteq \hat{\Phi}_i$ (\alg\ has added all strong parents).
\end{enumerate}

\noindent {\bf Case 1 (\alg\ has not added all of the strong parents):} ~\\
Let $k\in \Phi^s_i$ that has not been added yet ($k \not\in \hat{\Phi}_i$) such that $k$ verifies Assumption \ref{Ass:Coupling}, and $j$ be a non-parent variable. We know such a $k$ and corresponding realization-action pair which had been exhibited $N(w, a)$ times exist, since we assume there is at least one realization-action pair of the full parents with enough samples (since otherwise the requested bound holds trivially). We want to bound the probability that
\begin{equation} \label{eq:noWP1}
\widehat{\alpha}^*(k) - \widehat{\alpha}^*(j) > 0 \enspace ,
\end{equation}
where
$$
\widehat{\alpha}^*(k)- \widehat{\alpha}^*(j) =
\begin{array}{c}
 \max_{v,v_k,a} \| \widehat{\Pr}(\uY(i) | \uX(\hat{\Phi}_i )=v, a ) - \widehat{\Pr}(\uY(i) | \uX(\hat{\Phi}_i \cup \{ k \}) = (v,v_k), a ) \|_1 - \\
 \max_{v',v'_j,a'} \| \widehat{\Pr}(\uY(i) | \uX(\hat{\Phi}_i )=v', a' ) - \widehat{\Pr}(\uY(i) | \uX(\hat{\Phi}_i \cup \{ j \}) = (v',v'_j), a' ) \|_1 \enspace .
\end{array}
$$

If \eqref{eq:noWP1} holds for all non-parents, then \alg\ will only add parents from $\Phi_i$. For \eqref{eq:noWP1} to hold, it is sufficient to have
\begin{align} 
&\widehat{\alpha}^*(k) - \widehat{\alpha}(j,v,v_j,a) > 0 \qquad \forall j\in[D]\backslash\Phi_i,v,v_j,a
\enspace ,
\end{align}

	By applying the triangle inequality, we obtain
	\begin{equation} \label{eq:tri1}
	\begin{split}
	& {\alpha}(k) = \| \Pr(\uY(i) | \uX(\hat{\Phi}_i)=v, a ) - \Pr(\uY(i) | \uX(\hat{\Phi}_i ) = v, a ) \|_1 \leq 
	\\
	& \quad \| \widehat{\Pr}(\uY(i) | \uX(\hat{\Phi}_i\cup \{ k \})=(v,v_k), a ) - \Pr(\uY(i) | \uX(\hat{\Phi}_i \cup \{ k \}) = (v,v_k), a ) \|_1 + 
	\\
	& \quad \| \Pr(\uY(i) | \uX(\hat{\Phi}_i)=v, a ) - \widehat{\Pr}(\uY(i) | \uX(\hat{\Phi}_i ) = v, a ) \|_1 +
	\\
	&   \quad \widehat{\alpha}(k) \enspace ,
	\end{split}
	\end{equation}
	and
	\begin{equation}\label{eq:tri2}
	\begin{split}
	& \widehat{\alpha}(j) = \| \widehat{\Pr}(\uY(i) | \uX(\hat{\Phi}_i )=v, a ) - \widehat{\Pr}(\uY(i) | \uX(\hat{\Phi}_i \cup \{ j \}) = (v,v_j), a ) \|_1 \leq 
	\\
	& \quad \| \widehat{\Pr}(\uY(i) | \uX(\hat{\Phi}_i\cup \{ j \})=(v,v_j), a ) - \Pr(\uY(i) | \uX(\hat{\Phi}_i \cup \{ j \}) = (v,v_j), a ) \|_1 + 
	\\
	& \quad \| \Pr(\uY(i) | \uX(\hat{\Phi}_i)=v, a ) - \widehat{\Pr}(\uY(i) | \uX(\hat{\Phi}_i \cup \{ j \}) = (v,v_j), a ) \|_1 +
	\\
	& \quad \alpha(j) \enspace .
	\end{split}
	\end{equation}
	
	By applying equations \ref{eq:tri1} and \ref{eq:tri2}, Lemma \ref{Lem:Sample Complexity} (with our choice of $N(\epsilon,\delta_1)$) and Assumption \ref{Ass:Coupling}, 
	\begin{align*}
	\widehat{\alpha}^*(k) - \widehat{\alpha}(j,v,v_j,a) \geq& \\
	\begin{split}
	&  \alpha^*(k)
	\\
	& - \| \Pr(\uY(i) | \uX(\hat{\Phi}_i\cup \{ k \}) = (v^*,v^*_k), a^* ) - \widehat{\Pr}(\uY(i) | \uX(\hat{\Phi}_i \cup \{ k \}) = (v^*,v^*_k), a^* ) \|_1 
	\\
	& - \| \widehat{\Pr}(\uY(i) | \uX(\hat{\Phi}_i)=v^*, a^* ) - \Pr(\uY(i) | \uX(\hat{\Phi}_i ) = v^*, a^* ) \|_1 
	\\
	& \qquad - \| \Pr(\uY(i) | \uX(\hat{\Phi}_i\cup \{ j \}) = (v,v_j), a ) - \widehat{\Pr}(\uY(i) | \uX(\hat{\Phi}_i \cup \{ j \}) = (v,v_j), a ) \|_1 
	\\
	& \qquad - \| \widehat{\Pr}(\uY(i) | \uX(\hat{\Phi}_i)=v, a ) - \Pr(\uY(i) | \uX(\hat{\Phi}_i ) =v, a ) \|_1 
	\\
	& \qquad - \alpha(j,v,v_j,a)
	\\
	\geq& \quad \alpha^*(k) - \alpha(j,v,v_j,a) - 4 \epsilon \\
	\geq& \quad C_1 - 4 \epsilon  > 0
	\end{split}
	\end{align*}

	with probability at least $1-4\delta_1$ (union bound) for a particular $v,v_j,a$ if $C_1 > 4 \epsilon$. This holds for all $j,v,v_j,a$ with probability at least $1-(2+2(D-m)A\Gamma^{|\hat{\Phi}_i|+1})\delta_1$ (union bound again).
	
	This also means all variables in $\Phi^s_i$ will all be detected by \alg. Indeed, using the triangle inequality, the same bounds on
\begin{align}
	& \| \Pr(\uY(i) | \uX(\hat{\Phi}_i\cup \{ k \}) = (v^*,v^*_k), a^* ) - \widehat{\Pr}(\uY(i) | \uX(\hat{\Phi}_i \cup \{ k \}) = (v^*,v^*_k), a^* ) \|_1 
	\\
	& \| \widehat{\Pr}(\uY(i) | \uX(\hat{\Phi}_i)=v^*, a^* ) - \Pr(\uY(i) | \uX(\hat{\Phi}_i ) = v^*, a^* ) \|_1 
\end{align}
	 then above, the fact that (Assumption \ref{Ass:Coupling})
$$
\alpha^*(k) \geq \max_{j\in[D]\backslash\Phi_i}\alpha^*(j) + C_1 \geq C_1 \enspace ,
$$
and the fact that $C_1 > 4 \epsilon + C_2$ we have
\begin{align*}
	\widehat{\alpha}^*(k) \geq&
	\\
	\begin{split} 
	& \alpha^*(k)   
	\\
	& \quad - \| \Pr(\uY(i) | \uX(\hat{\Phi}_i\cup \{ k \}) = (v^*,v^*_k), a^* ) - \widehat{\Pr}(\uY(i) | \uX(\hat{\Phi}_i \cup \{ k \}) = (v^*,v^*_k), a^* ) \|_1 
	\\
	& \quad - \| \widehat{\Pr}(\uY(i) | \uX(\hat{\Phi}_i)=v^*, a^* ) - \Pr(\uY(i) | \uX(\hat{\Phi}_i ) = v^*, a^* ) \|_1 
	\\
	\end{split}	
	\\
	\geq& \qquad C_1 - 2\epsilon > C_2 + 2 \epsilon \enspace,
\end{align*}
with probabilities $1-2 \delta_1$.
	
	Notice that we made Assumption \ref{Ass:Coupling} much stronger than needed as we demanded the strong parent to stand out for all its possible realizations. For the proof, we only need to ensure that at least one realization verifying Assumption \ref{Ass:Coupling} is seen enough times to make sure a strong parent is preferred. Alternatively, we could modify assumption 1 to bound the probability of not acquiring enough samples for a particular realization that has a sufficiently large score. This would have negligible impact on the bounds of this lemma, the assumption would be weaker, but its presentation in the body of the paper would be more complex.
	
\noindent {\bf Case 2 (\alg\ has added all strong parent variables):} ~\\
	Now, we bound the probability that \alg\ adds a non parent variable $j$ if all strong parents variable $\Phi^s_i$ have already been added, that is, $\Phi^s_i \subseteq \hat{\Phi}_i \subseteq \Phi_i$: 
\begin{align*}
	\begin{split}
	\widehat{\alpha}(j) \leq& 
	\\
	& \quad \| \widehat{\Pr}(\uY(i) | \uX(\hat{\Phi}_i \cup \{ j \}) = (v,v_j), a ) - \Pr(\uY(i) | \uX(\hat{\Phi}_i \cup \{ j \}) = (v,v_j), a ) \|_1 + 
	\\
	& \quad \| \Pr(\uY(i) | \uX(\hat{\Phi}_i)=v, a ) - \widehat{\Pr}(\uY(i) | \uX(\hat{\Phi}_i ) =v, a ) \|_1 +
	\\
	& \quad \underbrace{\alpha(j)}_{\leq C_2 \textit{ because of Assumption \ref{Ass:dwarfNonParents} }}  \leq C_2 + 2\epsilon, 
	\end{split}
\end{align*}
	with probability at least $1-2\delta_1$ (according to Lemma \ref{Lem:Sample Complexity}) for a particular $v,v_j,a$ and with probability at least $1-2(D-m)A\Gamma^{|\hat{\Phi}_i|+1}\delta_1$ for all $v,v_j,a$.

\noindent {\bf Combining Case 1 \& 2:} ~\\
These two points must hold for all stages of the algorithm.\begin{itemize}
\item They must also hold for each iteration building $\hat\Phi_i$. Iterations in the first step correspond to all strong parents, and $\Phi_i^{w,1} \in \Phi_i^{w} $, the weak parents added in step 1 (before all strong parents are included). The number of iterations in the second point is at most all remaining weak parents $\Phi_i^{w,2} \subseteq \Phi_i^{w} \backslash \Phi_i^{w,1}$ added in the second step, plus one (when the algorithm stops). Note that the probability the first point holds is only $1-(2+2(D-m)A\Gamma^{|\hat{\Phi}_i|+1})\delta_1$ and not $1-(4+2(D-m)A\Gamma^{|\hat{\Phi}_i|+1})\delta_1$ because we are using the same two bounds involving $k$ twice.
\item They must hold for all $D$ target variable $i$.
\end{itemize}

Let $\kappa = 2(D-m)A\Gamma^{m+1}\geq 2(D-m)A\Gamma^{|\hat{\Phi}_i|+1}$. 
Using the union bound, these points hold for all stages of the algorithm with at least probability 
\begin{align}
1- \max_i\left((|\Phi_i^{s}| + |\Phi_i^{w,1}|)  (2+\kappa)+ (|\Phi_i^{w,2}|+1) \kappa \right) \delta_1 D &\geq 
1- (\max_i|\Phi_i| (2+\kappa) + \kappa)  D\delta_1 \\
&\geq 1- (2m + (m+1)\kappa)  D\delta_1 \\
&\geq 1- 2D \left[m+ (m+1)(D-m)A\Gamma^{m+1}\right]  \delta_1
\end{align}

\noindent {\bf Transitioning from Probabilities over $\hat{\Phi}_i$ to Probabilities over $\Phi_i$:} ~\\
	We define $\Phi^k_i$ to be the union of $\hat{\Phi}_i$ with the first $k$ variables in $\Phi_i \setminus \hat{\Phi}_i$ to be added greedily (according to the true probabilities) for the specific $(w, a)$ pair. Also, denote $w=(v,\bar{v}_1^{|\Phi_i \setminus \hat{\Phi}_i|} )$.%
	\begin{equation}
	\begin{split}
	& \| \Pr(\uY(i) | \uX ( \Phi_i ) = (v,\bar{v}), a) - \widehat{\Pr}(\uY(i) | \uX(\hat{\Phi}_i)=v, a) \|_1 
	\\	
	& \leq \sum_{k=1}^{|\Phi_i \setminus \hat{\Phi}_i|} \| \Pr(\uY(i) | \uX(\Phi^k_i) = (v, \bar{v}_1^k), a) - \Pr(\uY(i) | \uX(\Phi^{k-1}_i) = (v, \bar{v}_1^{k-1}), a) \|_1 
	\\	
	&  \;\;\;\;\;\; + \| \Pr(\uY(i) | \uX(\Phi_i)=v, a) - \widehat{\Pr}(\uY(i) | \uX(\hat{\Phi}_i)=v, a) \|_1 \enspace .
	\end{split}
	\end{equation}
	The inequality is due to the triangle inequality - we observe the quality of adding each additional parent, and are left with the estimation error on $v$. Since the parents were added greedily, by Assumption \ref{Ass:Submodularity} we can form a bound for the sum. Since we have enough samples of $v$ (it's in $\Theta_i$) the second term is small with high probability (by Lemma \ref{Lem:Sample Complexity}):	
	\begin{equation}
	\begin{split}	
	& \leq m \| \Pr(\uY(i) | \uX(\Phi^{1}_i)=(v, \bar{v}_1), a) - \Pr(\uY(i) | \uX(\hat{\Phi}_i)=v, a) \|_1 + m^2 C_3 + \epsilon \enspace ,
	\\
	& \leq m  \| \Pr(\uY(i) |  \uX(\Phi^{1}_i)=(v, \bar{v}_1), a) - \widehat{\Pr}(\uY(i) |   \uX(\Phi^{1}_i)=(v, \bar{v}_1), a) \|_1 
	\\
	& \;\;\;\;\;\; + m\| \Pr(\uY(i) | \uX(\hat{\Phi}_i)=v, a) - \widehat{\Pr}(\uY(i) | \uX(\hat{\Phi}_i)=v, a) \|_1
	\\
	& \;\;\;\;\;\; + m \| \widehat{\Pr}(\uY(i) |  \uX(\Phi^{1}_i)=(v, \bar{v}_1), a) - \widehat{\Pr}(\uY(i) | X(\hat{\Phi}_i)=v, a) \|_1 + m^2 C_3 + \epsilon \enspace .
	\end{split}
	\end{equation}
	Where the inequality holds from the triangle inequality. Similar to before, the first two summands can be bounded by $\epsilon$ with probability $1-\delta_1$. The third summands is bounded by the algorithm - since $\bar{v}_1$ was not added to $\hat{\Phi}_i$, and there were enough samples from it ($N(w, a) \geq N(\epsilon, \delta_1)$), it is necessarily smaller than the threshold $2\epsilon + C_2$, with probability $1-2\delta_1$ for a specific $i$ and $1-2D\delta_1$ for all of them. Therefore, the difference is bounded by: $(4m+1) \epsilon + m C_2 + m^2 C_3 $ for these probabilities.
	
\noindent {\bf Everything together:} ~\\
\begin{align}
\| \Pr(\uY(i) | \uX ( \Phi_i ) = (v,\bar{v}), a) - \widehat{\Pr}(\uY(i) | \uX(\hat{\Phi}_i)=v, a) \|_1 \leq (4m+1) \epsilon + m C_2 + m^2 C_3 
\end{align}
 with at least probability (union bound)
 \begin{align}
1- 2D \left[m+ (m+1)(D-m)A\Gamma^{m+1}\right]  \delta_1 - 2D \delta_1 =
1- 2D (m+1)\left[1+(D-m)A\Gamma^{m+1}\right]  \delta_1 
\end{align}	
which is lower bounded by $1- 2D (m+1)(D+1-m)A\Gamma^{m+1}  \delta_1 
$ or 
$1-2D(D+1)^2A\Gamma^{m+1} \delta_1
$

\end{proof}

\subsection{Proof of Theorem \ref{thm:Offline Policy Evaluation}}

\paragraph{Theorem 1.}\textit{	
Suppose Assumptions \ref{Ass:Coupling}, \ref{Ass:dwarfNonParents} and \ref{Ass:Submodularity} hold. Let $\frac{C_1}{4} > \epsilon + \frac{C_2}{4}, \epsilon > 0, \delta_1 > 0$, and $m = \max_{i\in[D]} |\Phi_i|$, then there exists 
$$
H(\epsilon, \delta_1) = O \left( 
\frac{\Gamma^2}{\delta_1\epsilon^2} \ln \left( \frac{\Gamma}{\delta_1}\right) 
 \right)
$$
such that if \alg\ is given $H$ trajectories, with probably at least $1-2A D(m+2)(D+1-m)  \Gamma^{m+1} \delta_1$, \alg\ returns an evaluation of $\ePi$ satisfying:
	\begin{equation}	
	\left| \nu - \tilde{\nu} \right| \leq \delta^* T + \epsilon^* D T^2 
	\end{equation}	
	where
	\begin{equation}
	\begin{split}
	& \epsilon^* = (4m+1) \epsilon + m C_2 + m^2 C_3 
	\\
	& \delta^* =  T \sum_{i=1}^D \psi_i A \Gamma^m \delta_1 
	\\
	& \psi_i = \max\limits_{(v,a) \in F_i} \frac{\sum_{t=1}^T \Pr(\uX_t(\Phi_i)=v, a_t=a | \ePi)}{\sum_{t=1}^T \Pr(\uX_t(\Phi_i)=v, a_t=a | \bPi)} \enspace.
	\end{split}
	\end{equation}
}
\begin{proof}

\begin{enumerate}
\item By Lemma \ref{lem:num_trajectories}, given 
$$H(\epsilon, \delta')
 \geq \frac{4AD\Gamma^m}{\delta'} \left( \frac{2\Gamma^2}{\epsilon^2} \ln \left( \frac{4AD\Gamma^{m+1}}{\delta'}\right) + \ln \left( \frac{2AD\Gamma^{m}}{\delta'} \right) \right)
$$  
trajectories there is a partition of $\Gamma$ into more (set $K$) and less likely $(v,a)$ pairs with probability at least $1-\delta'$. Pairs in set $K$ are seen at least $N(\epsilon, \delta_1)$ times. 
\item Since these pairs in $K$ are seen at least $N(\epsilon, \delta_1)$ times, Lemma \ref{lem:model_error} provides a bound on the estimation error on the conditionnal transition probabilities in the FMDP constructed by GSCOPE that holds with probability at least  $1- 2D (m+1)(D+1-m)A\Gamma^{m+1}  \delta_1$. 
\item This FMDP is therefore an $D\epsilon^*$-induce MDP with respect to the original MDP and $K$ (Definition \ref{def:eps_induced_mdp}, Proposition \ref{prop:BoundingFactoredDeviations} and Lemma \ref{lem:num_trajectories}).
\item Therefore, $A4(D\epsilon^*,\delta^*,\pi)$ is verified with probability at least (union bound on steps 1 and 2)
$$ 1-(1+(m+1)(D+1-m)\Gamma)\delta' \geq  1-(m+2)(D+1-m)\Gamma\delta'$$ 
for
\begin{itemize}
\item $\epsilon^*=(4m+1) \epsilon + m C_2 + m^2 C_3 $,
\item $\delta^*= T \sum_{i=1}^D \psi_i \delta'  / 2D$.
\end{itemize}
\item These values are then substituted into the simulation Lemma, and we replace $\delta' = 2 A D \Gamma^m \delta_1$ (equation \ref{eq:delta_1}) to obtain the specified result.
\end{enumerate}
\end{proof}

%% file: sim_lemma.tex
\subsection{The Simulation Lemma}

In this subsection, we derive a simulation lemma for MDPs, which essentially says that for a fixed policy two MDPs with similar transition probability distributions will result in similar value functions. Our simulation lemma differs from other simulation lemmas (e.g., \citealt{Kearns2002a,Kakade2003}) in that we only need the guarantee to hold for the target policy. To formalize what we mean by ``similar'' MDPs, we introduce the following assumption.

\begin{define} \label{def:induced_mdp}
Let $M = \langle S, A, P, R, \rho \rangle$ be an MDP and $K \subseteq S \times A$. $M$ and $K$ define an {\bf induced MDP} $M_{K} = \langle S, A, P_K, R_K, \rho \rangle$, where 
$$
P_K(\uY|\uX,a) = \left\{ \begin{array}{ll} 
	P(\uY|\uX,a) & \text{if } (\uX,a) \in K \\
	1 & \text{if } (\uX,a) \notin K \wedge \uY = \uX \\
	0 & \text{otherwise} 
\end{array} \right.
$$
and
$$
R_K(\uX,a) = \left\{ \begin{array}{ll}
	R(\uX,a) & \text{if } (\uX, a) \in K \\
	0 & \text{otherwise}
\end{array} \right. \enspace .
$$
\end{define}

\begin{define} \label{def:eps_induced_mdp}
Let $\epsilon > 0$, $M = \langle S, A, P, R, \rho \rangle$ be an MDP, and $K \subseteq S \times A$. An {\bf $\epsilon$-induced MDP} $\widehat{M} = \langle S, A, \widehat{P}, \widehat{R}, \rho \rangle$ with respect to $M$ and $K$, satisfies
$$
\begin{array}{l}
	\forall_{(\uX,a) \in K} \| P(\cdot|\uX,a) - \widehat{P}(\cdot|\uX,a)\|_1 \leq \epsilon \enspace ,\\
	\forall_{(\uX,a) \notin K} \forall_{\uY \in S} \widehat{P}(\uY|\uX,a) = P_{K}(\uY|\uX,a) \enspace , \text{and} \\
	\forall_{(\uX,a) \in S \times A} \widehat{R}(\uX,a) = R_K(\uX,a) \enspace .
\end{array}
$$
\end{define}

\begin{assume} \label{asm:similar_mdps} A\ref{asm:similar_mdps}$(\epsilon, \delta, \pi)$ : 
	Let $\epsilon > 0$, $\delta \in (0,1]$, $\pi$ be a policy, and $M = \langle S, A, P, R, \rho \rangle$. There exists an $\epsilon$-induced MDP $\widehat{M}$ with respect to $M$ and the subset of the state-action space $K \subseteq S \times A$, such that the probability of encountering a state-action pair that is not in $K$ while following $\pi$ in $M$ is small:
	\begin{align} \label{eqn:low_failure_prob}
		\Pr\left[ \exists_{t \in [T]} (\uX_t, a_t) \notin K \mid M, \pi \right] \leq \delta \enspace .
	\end{align}
\end{assume}

\begin{lemma} (Simulation Lemma)
	Suppose Assumption \ref{asm:similar_mdps} holds with A\ref{asm:similar_mdps}$(\epsilon, \delta, \pi)$, then 
	\begin{equation}
		\left| \tilde{\nu} - \nu \right| \leq \delta T  + \epsilon T^2 \enspace ,
	\end{equation}
	where $\tilde{\nu} = \rho^\top V^{\pi}_{\widehat{M}}$ and $\nu = \rho^\top V^{\pi}_{M}$.
\end{lemma}

\begin{proof}

\begin{align*}
| \nu - \tilde{\nu} | &= | \rho^\top V^{\pi}_{M} - \rho^\top V^{\pi}_{\widehat{M}} | \\
&= | \rho^\top V^{\pi}_{M} - (\rho^\top V^{\pi}_{M_K} - \rho^\top V^{\pi}_{M_K}) - \rho^\top V^{\pi}_{\widehat{M}} | && \text{Insert } 0 = (\rho^\top V^{\pi}_{M_K} - \rho^\top V^{\pi}_{M_K}) \\
&\leq | \rho^\top V^{\pi}_{M} - \rho^\top V^{\pi}_{M_K} | + | \rho^\top V^{\pi}_{M_K} - \rho^\top V^{\pi}_{\widehat{M}} | && \text{By the triangle inequality.} \\
&\leq \delta T + | \rho^\top V^{\pi}_{M_K} - \rho^\top V^{\pi}_{\widehat{M}} | && \text{By \eqref{eqn:low_failure_prob}.}\\
\end{align*}

We represent by $P^{\pi}_{M_K}, P^{\pi}_{\widehat{M}} \in \mathbb{R}^{S \times S}$ and $R \in \mathbb{R}^S$ the transition matrices and rewards induced by the policy $\pi$. For any matrix $A$, we denote by $\left\Vert A \right\Vert_p$ the $p$-induced matrix norm $\| \cdot \|$. Notice that:
\begin{align*}
\left\Vert P^{\pi}_{M_K} - P^{\pi}_{\widehat{M}} \right\Vert_\infty &= \max \limits _{1 \leq i \leq S} \sum _{j=1} ^n \left| P^{\pi}_{M_K}(s_j|s_i,\pi) - P^{\pi}_{\widehat{M}}(s_j|s_i,\pi) \right| && \text{Norm definition}
\\
&= \max \limits _{1 \leq i \leq S} \sum _{j=1} ^n \left| \sum_a \pi(a|s_i) (P_{M_K}(s_j|s_i,a) - P_{\widehat{M}}(s_j|s_i,a)) \right| && \text{Policy decomposition}
\\
&\leq \max \limits _{1 \leq i \leq S} \sum_a \pi(a|s_i) \sum _{j=1} ^n \left|   P_{M_K}(s_j|s_i,a) - P_{\widehat{M}}(s_j|s_i,a) \right| && \text{Triangle inequality}
\\
&\leq \max \limits _{1 \leq i \leq S} \sum_a \pi(a|s_i) \epsilon = \epsilon && \text{By Definition \ref{def:eps_induced_mdp}}
\end{align*}

In addition, we use the following result (page 254 in \citealt{bhatia1997matrix}): For any two matrices $X, Y$ and induced norm:
\begin{equation} \label{Eq:Bhatia}
\left\Vert X^m - Y^m \right\Vert \leq m M^{m-1} \left\Vert X - Y \right\Vert, 
\end{equation}
where $M = \max (\left\Vert X \right\Vert,  \left\Vert Y \right\Vert)$. Since $P^{\pi}_{M_K}, P^{\pi}_{\widehat{M}}$ are stochastic, this inequality holds for the $\infty$-induced norm with $M=1$. Now:

\begin{align*}
| \rho^\top V^{\pi}_{M_K} - \rho^\top V^{\pi}_{\widehat{M}} | & = \left| \rho^\top \sum_{t=0}^T (P^{\pi}_{M_K})^t R - \rho^\top \sum_{t=0}^T (P^{\pi}_{\widehat{M}})^t R \right| &&  \text{Sum of rewards over steps}
\\
& = \left| \rho^\top \left( \sum_{t=0}^T (P^{\pi}_{M_K})^t - \sum_{t=0}^T (P^{\pi}_{\widehat{M}})^t \right) R \right| && 
\\
& \leq \left\Vert \rho \right\Vert_1 \left\Vert \sum_{t=0}^T (P^{\pi}_{M_K})^t - \sum_{t=0}^T (P^{\pi}_{\widehat{M}})^t \right\Vert_\infty \left\Vert R \right\Vert_\infty &&  \text{H\"older inequality and submultiplicative norm}
\\
& \leq \sum_{t=0}^T \left\Vert (P^{\pi}_{M_K})^t - (P^{\pi}_{\widehat{M}})^t \right\Vert_\infty &&  \text{Triangle inequality and bounded reward}
\\
& \leq \left\Vert P^{\pi}_{M_K} - P^{\pi}_{\widehat{M}} \right\Vert_\infty \sum_{t=0}^T t && \text{Equation \ref{Eq:Bhatia} for each summand with $m=t$}
\\
& \leq  \epsilon T^2 && \text{Definition \ref{def:eps_induced_mdp} as seen above}
\end{align*}

Therefore, we can combine the results to obtain:

\begin{equation}
| \nu - \tilde{\nu} | \leq \delta T + \epsilon T^2
\end{equation}

\end{proof}

%% file: l1bound.tex
\subsection{Bounding the $L_1$-error in Estimates of the Transition Probabilities}

In this subsection, we consider the number of samples needed to estimate the transition probabilities of various realization-action pairs. The samples we receive are from a trajectory. Each trajectory is independent. Unfortunately, samples observed at timestep $t$ may depend on samples observed at previous timesteps. So the samples within a trajectory may not be independent. Therefore, we cannot apply the Weissman inequality \cite{Weissman2003}, which requires the samples to be independent and identically distributed. Instead, we derive a bound based on a martingale argument.

\begin{define}
A sequence of random variables $X_0, X_1, \dots$ is a {\bf martingale} provided that for all $i \geq 0$, we have
\begin{align}
\mathbb{E}\left[ | X_i | \right] &< \infty \enspace , \text{and} \\
\mathbb{E}\left[ X_{i+1} \mid X_0, X_1, X_2, \dots , X_i \right] &= X_i \enspace .
\end{align}
\end{define}

\begin{theorem} (Azuma's inequality)
Let $\varepsilon > 0$ and $X_1, X_2, \dots$ be a martingale such that $\left| X_{i+1} - X_i \right| < b_i$ for $i \geq 1$, then for all $m \geq 1$
\begin{equation}
\Pr \left[ | X_m - X_1 | \geq \varepsilon \right] \leq 2 \exp \left( \frac{-\varepsilon^2}{2\sum_{i=1}^{m} b_i} \right) \enspace .
\end{equation}
\end{theorem}

\begin{define}
Let $X_1, X_2, \dots, X_m$ be any set of random variables with support in $\Gamma$ and $f : \Gamma^{m} \rightarrow \mathbb{R}$ is a function. A {\bf Doob martingale} is the sequence
\begin{align*}
B_0 &= \mathbb{E}_{X_1, X_2, \dots, X_m} \left[ f(X_1, X_2, \dots, X_m) \right] \enspace , \text{ and} \\
B_i &= \mathbb{E}_{X_{i+1}, X_{i+2}, \dots, X_m} \left[ f(X_1, X_2, \dots, X_m) | X_1, X_2, \dots, X_{i} \right] \enspace , \text{ for } i = 1, 2, \dots, m \enspace .
\end{align*} 
\end{define}

\begin{lemma} \label{lem:l1_bound}
Let $\varepsilon > 0$, $\Gamma$ be a finite set, $\vec{X} = \langle X_1, X_2, \dots, X_m \rangle$ be a collection of $m \geq 1$ random variables with support in $\Gamma$ generated by an unknown process, and $f_x(\vec{X}) = \frac{1}{m} \sum\limits_{i=1}^{m} \mathbb{I}\{ X_i = x \}$ for all $x \in \Gamma$. We denote by $\mu(x) = \mathbb{E}\left[ f_x(\vec{X}) \right]$ for all $x \in \Gamma$. Then
\begin{equation} \label{eqn:fx_bound}
	\Pr \left[ | f_x(\vec{X}) - \mu(x) | \geq \varepsilon \right] \leq 2 \exp \left( \frac{-\varepsilon^2m}{2} \right) \enspace ,
\end{equation}
for all $x \in \Gamma$ and
\begin{equation} \label{eqn:l1_bound}
	\Pr \left[ \left\| \hat{\mu} - \mu \right\|_{1} \geq \varepsilon \right] \leq 2 |\Gamma| \exp \left( \frac{-\varepsilon^2m}{2|\Gamma|^2} \right) \enspace ,
\end{equation}
where $\hat{\mu}(x) = f_x(\vec{X})$.
\end{lemma}
\begin{proof}
First, notice that $\vec{X}$ and $m \cdot f_x(\cdot)$ define a Doob martingale such that $| B_{i+1} - B_{i} | \leq 1$ for $i = 1, 2, \dots m$. By applying Azuma's inequality, we obtain
\begin{align*}
\Pr \left[ | B_m - B_0 | \geq m \varepsilon \right] &\leq 2 \exp \left( \frac{-(\varepsilon m)^2}{2\sum_{i=1}^{m} 1} \right) \enspace ,\\
\Pr \left[ | f_x(\vec{X}) - \mu(x) | \geq \varepsilon \right] &\leq 2 \exp \left( \frac{-\varepsilon^2 m}{2} \right) \enspace , \\
\end{align*}
which proves \eqref{eqn:fx_bound}.

Now the union bound gives
\begin{align*}
\Pr \left[ \left\| \hat{\mu} - \mu \right\| \geq \sum_{x \in \Gamma} \frac{\varepsilon}{|\Gamma|} \right] \leq \sum_{x \in \Gamma} 2 \exp \left( \frac{-\varepsilon^2 m}{2|\Gamma|^2} \right) \enspace , \\
\end{align*}
which proves \eqref{eqn:l1_bound}.
\end{proof}

\begin{lemma} \label{Lem:Sample Complexity}
	Let $\epsilon, \delta > 0$, and $\Psi \subseteq [D]$, if there are 
$$
N \geq \frac{ 2\Gamma^2 }{\epsilon^2} \log \frac{2\Gamma}{\delta}
$$ samples of the realization-action pair $(v,a)$ obtained from independent trajectories of $\bPi$, then
	\begin{equation}
			\| \Pr(Y(i) | X(\Psi) = v, a)  - \widehat{\Pr}(Y(i)| X(\Psi) = v, a) \|_1 \leq \epsilon \enspace ,
	\end{equation}
with probability at least $1-\delta$.
\end{lemma}

\begin{proof}	
	Since the samples are taken from the behavior distribution, $\widehat{\Pr}(Y(i)=y | X(\Psi) = v, a)  = \frac{n(y, v, a)}{n(v, a)}=\frac{1}{n(v,a)} \sum_{k=1}^{N} \mathbb{I}\{ Y_k(i)=y,  X_k(\Psi) = v, a_k=a \}$. By Lemma \ref{lem:l1_bound}:
	\begin{equation}
	\begin{split}
	\Pr (\| \Pr(Y(i) | & X(\Psi) = v, a) - \widehat{\Pr}(Y(i) | X(\Psi) = v, a) \|_1 \geq \epsilon ) \leq 2|\Gamma| \exp \left(\frac{-N\epsilon^2}{2|\Gamma|^2} \right)
	\end{split}
	\end{equation}  
	Setting $\delta= 2|\Gamma| \exp(\frac{-N\epsilon^2}{2|\Gamma|^2})$ we obtain $N = \frac{2|\Gamma|^2}{\epsilon^2} \log \left( \frac{2|\Gamma|}{\delta} \right)$.
\end{proof}

%% file: OfflineDBNEval.bbl
\begin{thebibliography}{23}
\providecommand{\natexlab}[1]{#1}
\providecommand{\url}[1]{\texttt{#1}}
\expandafter\ifx\csname urlstyle\endcsname\relax
  \providecommand{\doi}[1]{doi: #1}\else
  \providecommand{\doi}{doi: \begingroup \urlstyle{rm}\Url}\fi

\bibitem[{Bellemare} et~al.(2013){Bellemare}, {Naddaf}, {Veness}, and
  {Bowling}]{Bellemare2013}
{Bellemare}, M.~G., {Naddaf}, Y., {Veness}, J., and {Bowling}, M.
\newblock The arcade learning environment: An evaluation platform for general
  agents.
\newblock \emph{Journal of Artificial Intelligence Research}, 47:\penalty0
  253--279, 06 2013.

\bibitem[Bottou et~al.(2013)Bottou, Peters, Qui\~{n}onero Candela, Charles,
  Chickering, Portugaly, Ray, Simard, and Snelson]{Bottou2013}
Bottou, L., Peters, J., Qui\~{n}onero Candela, J., Charles, D.~X., Chickering,
  D.~M., Portugaly, E., Ray, D., Simard, P., and Snelson, E.
\newblock Counterfactual reasoning and learning systems: The example of
  computational advertising.
\newblock \emph{Journal of Machine Learning Research}, 14\penalty0
  (1):\penalty0 3207--3260, 2013.

\bibitem[Brafman \& Tennenholtz(2002)Brafman and Tennenholtz]{Brafman2002}
Brafman, R.~I. and Tennenholtz, M.
\newblock {R-MAX} - a general polynomial time algorithm for near-optimal
  reinforcement learning.
\newblock \emph{Journal of Machine Learning Research}, 3:\penalty0 213--231,
  2002.

\bibitem[Chakraborty \& Stone(2011)Chakraborty and
  Stone]{chakraborty2011structure}
Chakraborty, D. and Stone, P.
\newblock Structure learning in ergodic factored mdps without knowledge of the
  transition function's in-degree.
\newblock In \emph{Proceedings of the 28th International Conference on Machine
  Learning (ICML-11)}, pp.\  737--744, 2011.

\bibitem[Degris et~al.(2006)Degris, Sigaud, and Wuillemin]{degris2006learning}
Degris, T., Sigaud, O., and Wuillemin, P.-H.
\newblock Learning the structure of factored {M}arkov decision processes in
  reinforcement learning problems.
\newblock In \emph{Proceedings of the 23rd international conference on Machine
  learning}, pp.\  257--264. ACM, 2006.

\bibitem[Dietterich(1998)]{Dietterich1998}
Dietterich, T.~G.
\newblock The {MAXQ} method for hierarchical reinforcement learning.
\newblock In \emph{Proceedings of the 15th International Conference on Machine
  Learning}, pp.\  118--126, 1998.

\bibitem[Diuk et~al.(2009)Diuk, Li, and Leffler]{diuk2009adaptive}
Diuk, C., Li, L., and Leffler, B.~R.
\newblock The adaptive k-meteorologists problem and its application to
  structure learning and feature selection in reinforcement learning.
\newblock In \emph{Proceedings of the 26th International Conference on Machine
  Learning}, pp.\  249--256. ACM, 2009.

\bibitem[Fonteneau et~al.(2010)Fonteneau, Murphy, Wehenkel, and
  Ernst]{Fonteneau2010}
Fonteneau, R., Murphy, S., Wehenkel, L., and Ernst, D.
\newblock Model-free monte carlo-like policy evaluation.
\newblock In \emph{Proceedings of the Thirteenth International Conference on
  Artificial Intelligence and Statistics (AISTATS 2010), JMLR W\&CP}, volume~9,
  pp.\  217--224, 2010.

\bibitem[Friedman et~al.(1998)Friedman, Murphy, and Russell]{Friedman:1998aa}
Friedman, N., Murphy, K., and Russell, S.
\newblock Learning the structure of dynamic probabilistic networks.
\newblock In \emph{Proceedings of the Fourteenth conference on Uncertainty in
  artificial intelligence}, pp.\  139--147. Morgan Kaufmann Publishers Inc.,
  1998.

\bibitem[Guestrin et~al.(2002)Guestrin, Patrascu, and
  Schuurmans]{guestrin2002algorithm}
Guestrin, C., Patrascu, R., and Schuurmans, D.
\newblock Algorithm-directed exploration for model-based reinforcement learning
  in factored mdps.
\newblock In \emph{Proceedings of the 19th International Conference on Machine
  Learning}, pp.\  235--242, 2002.

\bibitem[Guestrin et~al.(2003)Guestrin, Koller, Parr, and
  Venkataraman]{guestrin2003efficient}
Guestrin, C., Koller, D., Parr, R., and Venkataraman, S.
\newblock Efficient solution algorithms for factored mdps.
\newblock \emph{J. Artif. Intell. Res.(JAIR)}, 19:\penalty0 399--468, 2003.

\bibitem[Hester \& Stone(2009)Hester and Stone]{Hester2009}
Hester, T. and Stone, P.
\newblock Generalized model learning for reinforcement learning in factored
  domains.
\newblock In \emph{The Eighth International Conference on Autonomous Agents and
  Multiagent Systems (AAMAS)}, May 2009.

\bibitem[Jong \& Stone(2007)Jong and Stone]{Jong2007}
Jong, N.~K. and Stone, P.
\newblock Model-based function approximation in reinforcement learning.
\newblock In \emph{Proceedings of the 6th international joint conference on
  Autonomous agents and multiagent systems}, pp.\  95:1--95:8, 2007.

\bibitem[Kearns \& Koller(1999)Kearns and Koller]{kearns1999efficient}
Kearns, M. and Koller, D.
\newblock Efficient reinforcement learning in factored mdps.
\newblock In \emph{IJCAI}, volume~16, pp.\  740--747, 1999.

\bibitem[Kearns \& Singh(2002)Kearns and Singh]{Kearns2002a}
Kearns, M. and Singh, S.
\newblock Near-optimal reinforcement learning in polynomial time.
\newblock \emph{Machine Learning}, 49\penalty0 (2):\penalty0 209--232, 2002.

\bibitem[{Li} et~al.(2014){Li}, {Munos}, and {Szepesvari}]{Li2014}
{Li}, L., {Munos}, R., and {Szepesvari}, C.
\newblock {On Minimax Optimal Offline Policy Evaluation}.
\newblock \emph{ArXiv e-prints}, September 2014.

\bibitem[Precup(2000)]{precup2000eligibility}
Precup, D.
\newblock Eligibility traces for off-policy policy evaluation.
\newblock \emph{Computer Science Department Faculty Publication Series}, pp.\
  ~80, 2000.

\bibitem[Puterman(2009)]{puterman2009markov}
Puterman, M.~L.
\newblock \emph{Markov decision processes: discrete stochastic dynamic
  programming}, volume 414.
\newblock John Wiley \& Sons, 2009.

\bibitem[Richardson et~al.(2007)Richardson, Dominowska, and
  Ragno]{richardson2007predicting}
Richardson, M., Dominowska, E., and Ragno, R.
\newblock Predicting clicks: estimating the click-through rate for new ads.
\newblock In \emph{Proceedings of the 16th international conference on World
  Wide Web}, pp.\  521--530. ACM, 2007.

\bibitem[Strehl et~al.(2007)Strehl, Diuk, and Littman]{Strehl2007}
Strehl, A.~L., Diuk, C., and Littman, M.~L.
\newblock Efficient structure learning in factored-state {MDP}s.
\newblock In \emph{Proceedings of the Twenty-Second Conference on Artificial
  Intelligence (AAAI-07)}, volume~7, pp.\  645--650, 2007.

\bibitem[Sutton \& Barto(1998)Sutton and Barto]{Sutton1998}
Sutton, R. and Barto, A.
\newblock \emph{{Reinforcement Learning: An Introduction}}.
\newblock MIT Press, 1998.

\bibitem[Thomas et~al.(2015)Thomas, Theocharous, and
  Ghavamzadeh]{Theocharous2015offPolicyConfidence}
Thomas, P.~S., Theocharous, G., and Ghavamzadeh, M.
\newblock High confidence off-policy evaluation.
\newblock In \emph{Proceedings of the Twenty-Ninth Conference on Artificial
  Intelligence}, 2015.

\bibitem[Trabelsi et~al.(2013)Trabelsi, Leray, Ben~Ayed, and
  Alimi]{Trabelsi:2013aa}
Trabelsi, G., Leray, P., Ben~Ayed, M., and Alimi, A.
\newblock Dynamic {MMHC}: A local search algorithm for dynamic {B}ayesian
  network structure learning.
\newblock In \emph{Advances in Intelligent Data Analysis XII}, volume 8207 of
  \emph{Lecture Notes in Computer Science}, pp.\  392--403. Springer Berlin
  Heidelberg, 2013.

\end{thebibliography}


\begin{thebibliography}{6}
\providecommand{\natexlab}[1]{#1}
\providecommand{\url}[1]{\texttt{#1}}
\expandafter\ifx\csname urlstyle\endcsname\relax
  \providecommand{\doi}[1]{doi: #1}\else
  \providecommand{\doi}{doi: \begingroup \urlstyle{rm}\Url}\fi

\bibitem[Bhatia(1997)]{bhatia1997matrix}
Bhatia, R.
\newblock \emph{Matrix analysis}, volume 169.
\newblock Springer Science \& Business Media, 1997.

\bibitem[Kakade(2003)]{Kakade2003}
Kakade, S.~M.
\newblock \emph{On the Sample Complexity of Reinforcement Learning}.
\newblock PhD thesis, University College London, March 2003.

\bibitem[Kearns \& Singh(2002)Kearns and Singh]{Kearns2002a}
Kearns, M. and Singh, S.
\newblock Near-optimal reinforcement learning in polynomial time.
\newblock \emph{Machine Learning}, 49\penalty0 (2):\penalty0 209--232, 2002.

\bibitem[Li(2009)]{Li2009}
Li, L.
\newblock \emph{A Unifying Framework for Computational Reinforcement Learning
  Theory}.
\newblock PhD thesis, Rutgers University, 2009.

\bibitem[Osband \& Van~Roy(2014)Osband and Van~Roy]{Osband:2014aa}
Osband, I. and Van~Roy, B.
\newblock Near-optimal reinforcement learning in factored mdps.
\newblock In \emph{Advances in Neural Information Processing Systems}, pp.\
  604--612, 2014.

\bibitem[Weissman et~al.(2003)Weissman, Ordentlich, Seroussi, Verdu, and
  Weinberger]{Weissman2003}
Weissman, T., Ordentlich, E., Seroussi, G., Verdu, S., and Weinberger, M.~J.
\newblock Inequalities for the l1 deviation of the empirical distribution.
\newblock Technical report, Hewlett-Packard Labs, 2003.

\end{thebibliography}
